\documentclass{article}

\usepackage[utf8]{inputenc} 
\usepackage[T1]{fontenc}    
\usepackage{hyperref}       
\usepackage{url}            
\usepackage{booktabs}       
\usepackage{amsfonts}       
\usepackage{nicefrac}       
\usepackage{microtype}      
\usepackage{amsthm}
\usepackage{fullpage}
\usepackage{hyperref}
\usepackage{xcolor}
\usepackage{xspace}
\usepackage{graphics}
\usepackage{graphicx}
\usepackage{subcaption}
\usepackage{apxproof}
\graphicspath{ {./images/} }
\usepackage{bbm}
\usepackage{amsmath}
\usepackage{amsfonts}
\usepackage{algorithm}
\usepackage{algpseudocode}
\usepackage{mathtools}
\usepackage{wrapfig}
\usepackage[normalem]{ulem}
\usepackage{authblk}

\theoremstyle{plain}
\newtheorem{theorem}{Theorem}[section]
\newtheorem{proposition}[theorem]{Proposition}
\newtheorem{lemma}[theorem]{Lemma}
\newtheorem{corollary}[theorem]{Corollary}
\theoremstyle{definition}
\newtheorem{definition}[theorem]{Definition}

\theoremstyle{remark}

\usepackage{hyperref}

\usepackage{amsmath}
\usepackage{amssymb}
\usepackage{mathtools}
\usepackage{amsthm}

\usepackage[capitalize,noabbrev]{cleveref}

\usepackage{bbm}

\newcommand{\E}{\mathbb{E}}

\def\cN{\mathcal{N}}

\usepackage{float}

\title{Improved Differentially Private Regression via Gradient Boosting}
\author[1]{Shuai Tang\footnote{Shuai Tang is the lead author; all other authors are listed in alphabetical order. shuat@amazon.com}}
\author[1]{Sergul Aydore}
\author[1,3]{Michael Kearns}
\author[2]{Saeyoung Rho}
\author[1,3]{Aaron Roth}
\author[1]{\\Yichen Wang}
\author[1,5]{Yu-Xiang Wang}
\author[1,4]{Zhiwei Steven Wu}
\affil[1]{Amazon AWS AI/ML}
\affil[2]{Columbia University}
\affil[3]{University of Pennsylvania}
\affil[4]{Carnegie Mellon University}
\affil[5]{University of California, Santa Barbara}

\begin{document}
\date{}
\maketitle

\begin{abstract}
We revisit the problem of differentially private squared error linear regression. We observe that existing state-of-the-art methods are sensitive to the choice of hyperparameters --- including the ``clipping threshold'' that cannot be set optimally in a data-independent way. We give a new algorithm for private linear regression based on gradient boosting. We show that our method consistently improves over the previous state of the art when the clipping threshold is taken to be fixed without knowledge of the data, rather than optimized in a non-private way --- and that even when we optimize the hyperparameters of competitor algorithms non-privately, our algorithm is no worse and often better. In addition to a comprehensive set of experiments, we give theoretical insights to explain this behavior. 
\end{abstract}

\section{Introduction}
Squared error linear regression is a 
basic, 
foundational method in statistics and machine learning. Absent other constraints, it has an optimal closed-form solution. A consequence of this is that linear regression parameters have a deterministic relationship with the data they are fitting, which can leak private information. As a result, there is a substantial body of work aiming to approximate the solution to least squares linear regression with the protections of differential privacy \cite{reg0,reg1,reg2,wang2018revisiting,reg3,reg4,amin2022easy}.

We highlight the AdaSSP (``Adaptive Sufficient Statistics Perturbation'')  algorithm \cite{wang2018revisiting} which obtains state-of-the-art theoretical and practical performance when the maximum norm of the features and labels are known---these bounds are used to scale the noise added for privacy. When a data-independent bound on the magnitude of the data is not known, in order to promise differential privacy, they must be clipped at some data-independent threshold, which can substantially harm performance. In this work, we give a new algorithm for private linear regression that substantially mitigates this issue and leads to improved accuracy across a range of datasets and clipping thresholds.

Our approach is both conceptually and computationally simple: we apply gradient boosting \cite{friedman2001greedy}, using a linear model as the base learner, and to incorporate privacy guarantees, at each boosting round, the linear model is solved using AdaSSP. When applied to a squared error objective, gradient boosting is exceedingly simple: it maintains a linear combination of regression models, repeatedly fitting a new regression model to the \emph{residuals} of the current model, and then adding the new model to the linear combination. Absent privacy constraints, gradient boosting for linear regression does not improve performance, 
because linear models are closed under linear combinations, and squared error regression can be optimally solved over the set of all linear models in closed form. 
Nevertheless, in the presence of privacy constraints and in the absence of knowledge of the data scale (so that we must use a data independent clipping threshold), we show in an extensive set of experiments that gradient BoostedAdaSSP substantially improves on the performance of AdaSSP alone. Moreover, we show that our BoosedAdaSSP algorithm outperforms other competitive differentially private solutions to linear regression in different conditions, including gradient descent on the squared loss objective, and interestingly performs better than a tree-based private boosting algorithm. We also show that our algorithm is less sensitive to hyperparameter selection.

We also provide stylized theoretical explanations of the empirical results. In the zero-dimensional case, AdaSSP reduces to computing the empirical mean of the clipped data, and aggressive clipping thresholds can cause the bias of empirical mean to be arbitrarily large.
In this setting, gradient boosting with AdaSSP as a base learner corresponds to iteratively updating an estimator of the mean by the clipped empirical residuals
, i.e. the empirical mean of the difference between the current mean estimate and the data. In Section \ref{ss.clipping theory}, we show that, for Gaussian data, the boosting method converges to the true mean for \textit{any} non-zero clipping threshold.
The intuition behind this improvement of boosting over the one-shot empirical mean is that, even clipped estimates of the mean are directionally correct, which serves to further de-bias the current estimate and reduce the negative effect of aggressive clipping. 
The convergence of our boosted algorithm under arbitrary clipping provides a significant improvement over AdaSSP, especially when the clipping bound must be independent to the data.

Finally, we show that BoostedAdaSSP can sometimes out-perform differentially private boosted trees \cite{nori2021dpebm} as well, a phenomenon that we do not observe absent privacy. This contributes to an important conceptual message: that the best learning algorithms under the constraint of differential privacy are not necessarily ``privatized'' versions of the best learning algorithms absent privacy---differential privacy rewards algorithmic simplicity.

\subsection{Additional Related Work}
Because of its fundamental importance, linear regression has been the focus of a great deal of attention in differential privacy \cite{kifer2012private,reg0,reg1,reg2,wang2018revisiting,reg3,reg4,amin2022easy}, using techniques including private gradient descent \cite{bassily2014private,abadi2016deep}, output and objective perturbation \cite{chaudhuri2011differentially}, and perturbation of sufficient statistics \cite{vu2009differential}. As already mentioned, the AdaSSP (a variant of the sufficient statistic perturbation approach) \cite{wang2018revisiting} has stood out as a method obtaining both optimal theoretical bounds and strong empirical performance --- both under the assumption that the magnitude of the data is known. 

\cite{amin2022easy} have previously noted that AdaSSP can perform poorly when the data magnitude is unknown and clipping bounds must be chosen in data-independent ways. They also give a method --- TukeyEM \cite{amin2022easy} --- aiming to remove these problematic hyperparameters for linear regression. TukeyEM privately aggregates multiple non-private linear regressors learned on disjoint subsets of the training set. The private aggregate uses the approximate Tukey depth and removes the risk of potential privacy leaks in choosing hyperparameters. However, because each model is trained on a different partition of the data, 
as \cite{amin2022easy} note,  TukeyEM performs well when the number of samples is roughly $1,000$ times larger than the dimension of the data. We include a comparison to both TukeyEM and AdaSSP in our experimental results. 

Another line of work has studied differentially private gradient boosting methods, generally using a weak learner class of classification and regression trees (CARTs) 
\cite{li2020dpboost,grislain2021dpxgb}. \cite{nori2021dpebm} gives a particularly effective variant called DP-EBM, which we compare to in our experiments.

There is a line of work that aims to privately optimize hyperparameters (e.g. \cite{hyperparam2,hyperparam1,hyperparam3}) --- we do not directly compare to these approaches, but our experiments show that our algorithm dominates comparison methods even when their hyperparameters are optimized non-privately.

\section{Preliminaries}

We study the standard squared error linear regression problem. Given a joint distribution $\mathcal{D}$ over $p$ dimensional features $x \in \mathbb{R}^p$ and real-valued labels $y \in \mathbb{R}$. Our goal is to learn a parameter vector 
$\theta \in \mathbb{R}^p$ to minimize squared error:
\begin{align}
\mathcal{L}(\theta,\mathcal{D}) =  \mathbb{E}_{(x,y) \sim \mathcal{D}}[(\langle \theta, x \rangle - y)^2].
\end{align}
In order to protect privacy of individuals in the training data when the learnt parameter vector $\theta$ is released, 
we adopt the notion of Differential Privacy. 

\subsection{Differential Privacy (DP)}\label{s.dp}
Differential privacy 
is a strong formal notion of individual privacy. 
DP ensures that, for a randomized algorithm, when two neighboring datasets that differ in one data point are presented, the two outputs are indistinguishable, within some probability margin defined using $\epsilon$ and $\delta \in [0,1)$.
\begin{definition}[Differential Privacy \cite{DMNS06}]
\label{def.dp}
A randomized algorithm $\mathcal{M}$ with domain $\mathcal{D}$ is $(\epsilon, \delta)$-differentially private for all $\mathcal{S} \subseteq \mbox{Range}(\mathcal{M})$ and for all pairs of neighboring databases $D,D' \in \mathcal{D}$,
\begin{align}
\Pr[\mathcal{M}(D) \in \mathcal{S} ] \leq e^\epsilon
\Pr[\mathcal{M}(D') \in \mathcal{S} ] + \delta,
\end{align}
where the probability space is over the randomness of the mechanism $\mathcal{M}$.
\end{definition}
A refinement of differential privacy, a single-parameter privacy definition (Gaussian differential privacy, GDP) was later proposed \cite{dong2021gaussian}. In this work, we use GDP in order to achieve better privacy bounds. We present several key results in \cite{dong2021gaussian} that we use in our privacy analysis.
\begin{definition}[$\ell_2$-sensitivity]\label{d.l2sens}
The $\ell_2$-sensitivity of a statistic $m$ over the domain of dataset $D$ is 
$\Delta(m)=\sup_{D,D^\prime}\|m(D)-m(D^\prime)\|_2$, 
where $\|\cdot\|_2$ is the vector $\ell_2$-norm, and the supremum is over all neighboring datasets.
\end{definition}
\begin{theorem}
[Gaussian Mechanism, Theorem 2.7 from of \cite{dong2021gaussian}]
\label{t.gm}
Define a randomized algorithm $GM$ that operates on a statistic $m$ as $GM(x, \mu)=m(x)+\eta$, where $\eta\sim \mathcal{N}(0, \Delta(m)^2/\mu^2)$ and $\Delta(m)$ is the $\ell_2$-sensitivity 
of the statistics $m$. Then, $GM$ is $\mu$-GDP. 
\end{theorem}
For $n$ GDP mechanisms with privacy parameters $\mu_1, \cdots, \mu_n$ ,the following composition theorem holds:
\begin{corollary}[Composition of GDP, Corollary 3.3 of \cite{dong2021gaussian}]
\label{c.composition}
The $n$-fold composition of $\mu_i$-GDP mechanisms is $\sqrt{\mu_1^2 + \cdots + \mu_n^2}$-GDP.
\end{corollary}
There is a tight relationship between $\mu$-GDP and $(\epsilon, \delta)$-DP that allows us to perform our analysis using GDP, and state our results in terms of $(\epsilon,\delta)$-DP. 
\begin{corollary}[Conversion between GDP and DP, Corollary 2.13 of \cite{dong2021gaussian}]
\label{t.conversion}
A mechanism is $\mu$-GDP if and only if it is $(\epsilon, \delta(\epsilon))$-DP for all $\epsilon \geq 0$, such that
\begin{align}
\delta(\epsilon) = \Phi \left( -\frac{\epsilon}{\mu} +\frac{\mu}{2} \right) - e^{\epsilon} \Phi \left(-\frac{\epsilon}{\mu} -\frac{\mu}{2} \right)
\end{align}
where $\Phi$ denotes the standard Gaussian CDF.
\end{corollary}

\section{Improved AdaSSP via Gradient Boosting}
Our algorithm for private linear regression uses gradient boosting with AdaSSP as a weak learner.
\subsection{Gradient Boosting}
\label{ss.loss}
For regression tasks, we assume that we have a dataset $D=\{x_i,y_i\}_{i=1}^n$, where $x_i\in\mathbb{R}^p$ and $y_i\in\mathbb{R}$, $\forall i \in [n]$. 
Let  $T$ be the number of boosting rounds, and $f_t$ be the model obtained at iteration $t\in[T]$. 
Since our base learner is linear and the objective is the squared loss, at the $t$-th round, the objective of a gradient boosting algorithm is to obtain:
\begin{align}
    \theta_{t}= \arg\min_{\theta}\sum_{i=1}^n (y_i - (\sum_{k=1}^{t-1}\theta_k^\top x_i + \theta^\top x_i))^2
                = \arg\min_{\theta}\sum_{i=1}^n\left(g_{i,t} - \theta^\top x_i\right)^2 \label{eq.objective_round},
\end{align}
where $g_{i,t}=y_i - \sum_{k=1}^{t-1}\theta_k^\top x_i$ is the steepest gradient of the objective function w.r.t. the ensemble predictions made by previous rounds.
Therefore, each gradient boosting round is  solving a squared error linear regression problem where the features are data, and the labels are gradients. The model update at $t$-th round is simply $\hat{\theta} = \hat{\theta} + \theta_{t}$, and the final model is $\hat{\theta}=\sum_{k=1}^T\theta_k$.

Since the update preserves the linearity of the model, and squared error regression can be solved optimally over linear models. Absent privacy, gradient boosting cannot improve the error of linear regression in the standard setting. Nevertheless, when we replace exact linear regression with differentially private approximations, the situation changes.

\subsection{Private Ridge Regression as a Base Learner}
\label{ss.base}

Let $X \in \mathbb{R}^{n \times p}$ be the matrix with $x_i$'s in each row and $g_t \in \mathbb{R}^n$ be a vector containing gradients of training samples at $t$ (i.e., $g_{i,t}$).
Absent privacy, there exists a closed-form soluion to Eq. \ref{eq.objective_round}, and it is
\begin{align}
 \theta_t= (X^\top X)^{-1} X^\top g_t ,
\end{align}

To provide differential privacy guarantees, 
AdaSSP (Algorithm 2 of \cite{wang2018revisiting}) is applied to learn a private linear model at each round. It also requires us to adjust our solution at each round from OLS to Ridge Regression as follows:
\begin{align}
 \theta_t= (X^\top X+\lambda I)^{-1} X^\top g_t ,
\end{align}
where $\lambda$ controls the strength of regularization, and $I\in\mathbb{R}^{p\times p}$ is the identity matrix.

Let $\mathcal{X}$ and $\mathcal{Y}$ be the domain of our features and labels, respectively. We define  bounds on the data domain $||\mathcal{X}|| = \sup_{x\in\mathcal{X}} ||x||$ and $||\mathcal{Y}|| = \sup_{y\in\mathcal{Y}} ||y||$.
Given as input privacy parameters $\epsilon$ and $\delta$, and bounds on the data scale $||\mathcal{X}||$ and $||\mathcal{Y}||$ for $x_i$ and $g_{i,t}$, AdaSSP chooses a noise scale  to obtain $\mu$-GDP for the appropriate value of $\mu$, and adds calibrated Gaussian noise to three sufficient statistics: 1) $X^\top X$, 2) $X^\top g_t$, and 3) $\lambda$. The adaptive aspect of AdaSSP comes from the fact that $\lambda$ is chosen based on $X^\top X$, therefore, we also need to allocate privacy budget for computing  $\hat{\lambda}$.
Details of the AdaSSP algorithm for learning one ridge regressor are deferred to Appendix  \ref{a.adasspalgo}.

Let 
$\widehat{X^\top X} = GM(X^\top X, \mu_1)$, 
$\widehat{X^\top g_t} = GM(X^\top g_t, \mu_2)$,
$\widehat{\lambda} = GM(\lambda_{\min} (X^\top X), \mu_3)$ 
be the private release of sufficient statistics from a single instantiation of AdaSSP to learn $\theta_t$ and $GM$ as defined in Theorem \ref{t.gm}.
The final model $\widehat{\theta}$ can be expressed as
\begin{align}
   \widehat{\theta} =  \textstyle \sum_{t=1}^T \widehat{\theta_t} = \left(\widehat{X^\top X} +\widehat{\lambda} I \right)^{-1} \sum_{t=1}^T \widehat{X^\top g_t}
\end{align}
Therefore, when running gradient boosting, we only need to release $GM(X^\top X,\mu_1)$ and $GM(\lambda_{\min}(X^\top X), \mu_3)$ once at the beginning of our algorithm, and at each stage, the only additional information we need to  release is $GM(X^\top g_t, \mu_2/\sqrt{T})$; this provides a savings over  naively repeating AdaSSP (given as Algorithm \ref{a.adassp} in the Appendix) for $T$ rounds.

Putting it all together, our final algorithm BoostedAdaSSP  is shown in Algorithm \ref{a.dpgb_linear}, and the privacy guarantee is shown in Theorem \ref{t.boostedadassp_gdp}.
\begin{theorem}
\label{t.boostedadassp_gdp}

Algorithm \ref*{a.dpgb_linear} satisfies $(\epsilon, \delta)$-DP.
\end{theorem}
\begin{proof}
We use the privacy of the Gaussian mechanism, and the composition theorem stated in corollary \ref{c.composition}, which gives us a GDP bound of: 
$\sqrt{\mu_1^2 + T \left( \mu_2/\sqrt{T}\right)^2 + \mu_3^2 }
= \sqrt{\mu_1^2 + \mu_2^2 + \mu_3^2 }
= \mu$. The conversion from GDP to DP follows from Corollary \ref{t.conversion}.
\end{proof}
\begin{algorithm}[t]
\caption{BoostedAdaSSP}
\label{a.dpgb_linear}
\begin{algorithmic}

\State \textbf{Input} Dataset $D=\{X,y\}$, privacy parameters $\epsilon$, $\delta$, split ratio $a, b, c$, and clipping threshold $\tau$ 
\State \textbf{Initialize} $\theta = 0$ 
\State Find $\mu$ such that $\mu$-GDP satisfies  $(\epsilon, \delta)$-DP.
\State Calibrate $\mu_1, \mu_2, \mu_3$ such that $\mu_1: \mu_2: \mu_3 = a:b:c$ and $\mu = \sqrt{\mu_1^2 + \mu_2^2 + \mu_3^2}$
\State Clip the norm of samples $x_i$ = clip($x_i$, 1), $\forall i\in[n]$ 

\State Private release of $\widehat{\lambda} = GM(\lambda_{\min}(X^\top X), \mu_3)$ and $\widehat{X^\top X} = GM(X^\top X, \mu_1)$ 
\State Compute $\Gamma = (\widehat{X^\top X} +\widehat{\lambda} I)^{-1} $
\For{$t \in [T]$}

    \State $g_t = y - X\theta$ \texttt{\small \# negative gradient}
    \State $g_{i,t} = \text{clip}(g_{i,t}, \tau)$, $\forall i\in[n]$ \texttt{\small \# gradient clipping}
    \State $\widehat{X^\top g_t}=GM(X^\top g_t, \frac{\mu_2}{\sqrt{T}})$ \texttt{\small \# private release}
    \State $\theta_t =  \Gamma \widehat{X^\top g_t}\;$ \texttt{\small \# private linear model}
    \State $\;\theta \leftarrow \theta + \theta_t$ \texttt{\small \# model update}
\EndFor
\State \textbf{Return} $\theta$
\State * $\text{clip}(x, \tau) = x\min(1,\tau/||x||)$
\State * $GM(X, \mu)$ denotes the Gaussian mechanism that satisfies $\mu$-GDP for releasing noisy $X$ 

\end{algorithmic}
\end{algorithm}

\subsection{Data-independent Clipping Bounds}
As described in \cite{wang2018revisiting} and mentioned in \cite{amin2022easy}, the clipping bounds on $\mathcal{X}$ and $\mathcal{Y}$ are taken to be known --- but if they are selected as a deterministic function of the data, this would constitute a violation of differential privacy. 
For $\mathcal{Y}$, the most natural solution is to use a data-independent $\tau$ to clip labels and enforce a bound of $\tau$; but as we observe both empirically and theoretically, this introduces a difficult-to-tune hyperparameter that can lead to a substantial degradation in performance. 
For $\mathcal{X}$, one way to resolve this issue, (as is done in the implementation of AdaSSP
\footnote{\href{https://github.com/yuxiangw/autodp/blob/master/\newline tutorials/tutorial_AdaSSP_vs_noisyGD.ipynb}{https://github.com/yuxiangw/autodp/blob/master/tutorials/tutorial\_AdaSSP\_vs\_noisyGD.ipynb}})
is to normalize each individual data point to have norm $1$, but this is not without loss of generality: it  fundamentally changes the nature of the regression problem being solved, and so does not always constitute a meaningful solution to the original problem. Instead, we clip the norm of data points so that the maximum norm doesn't exceed a fixed data independent threshold (but might be lower). 

\section{Experiments}

We selected 33 tabular datasets with single-target regression tasks from OpenML \footnote{\href{https://www.openml.org}{https://www.openml.org}} \cite{grinsztajn2022tree} for evaluating and comparing our algorithm to other algorithms.
Task details are presented in Table  \ref{app.regression_tasks} .
The selected tasks include both categorical and numerical features. We assume that the schema  of individual tables is public information, and so convert categorical features into one-hot encodings.

We compare our approach with a number of other algorithms. First, we compare to other private linear regression methods: AdaSSP,  DP Gradient Descent and TukeyEM. These represent  the leading practical methods (with accompanying code) used for solving linear regression problems. DP Gradient Descent solves the linear regression problem through noisy batch gradient descent with noise calibrated with clipped per-sample gradients, meanwhile, TukeyEM trains nonprivate linear models on disjoint subsets and privately aggregates the learned linear models. Since our algorithm is based on gradient boosting, in addition to algorithms that solve linear regression problems, we also compare to  DP-EBM \footnote{\href{https://github.com/interpretml/interpret}{https://github.com/interpretml/interpret}},
 the current state-of-the-art differentially private gradient boosting algorithm, which uses  trees as its base learners. Rather than finding the optimal splits for each leave based on the data, DP-EBM uses random splits, which significantly improves the efficacy of the privacy budget.

As each algorithm has it own hyperparameters (which are often tuned non-privately in reported results),  we present three sets of comparisons. 1) First, we compare performance of the algorithms when the hyperparameters are non-privately optimized for each dataset, for each of the algorithms. This provides an (unrealistically) optimistic view of each algorithm's best case performance. 2) Next, we use a fixed set of hyperparameters for our algorithm (BoostedAdaSSP), which remain unchanged from dataset to dataset, while still non-privately optimizing the hyperparameters of each of our comparison partners on a dataset-by-dataset basis. This provides an (unrealistic) best-case comparison for the methods we benchmark against. 3) Finally, we show what we view as the fair comparison, which is when the hyperparameters of our method (BoostedAdaSSP) as well as those of all of our comparison partners are held constant across all of the datasets.   For hyperparameter tuning, Optuna \cite{optuna_2019} is applied. The tuning ranges of hyperparameters, and the fixed hyperparameters for our method are reported in Table \ref{tab:hyperparams} in the Appendix. For each comparison partner, when we fix the parameters, we use parameters recommended in their papers.

\textbf{Gradient Boosting Improves AdaSSP.} When hyperparameters are non-privately tuned for both methods, then the mean squared error is quite similar on most datasets for both methods, but our method (BoostedAdaSSP) obtains lower error on the majority of datasets at all tested privacy levels. When BoostedAdaSSP uses fixed hyperparameters, it remains competitive with AdaSSP even when AdaSSP is non-privately tuned on each dataset. Finally, when both methods use fixed hyperparameters, BoostedAdaSSP has substantially improved error across a majority of datasets at all privacy levels. This indicates a substantial advantage for our method.
Comparisons are presented in Fig. \ref{fig:adassp}.

\begin{figure}[t]
    \centering
    \begin{subfigure}[b]{\textwidth}
        \includegraphics[width=\linewidth]{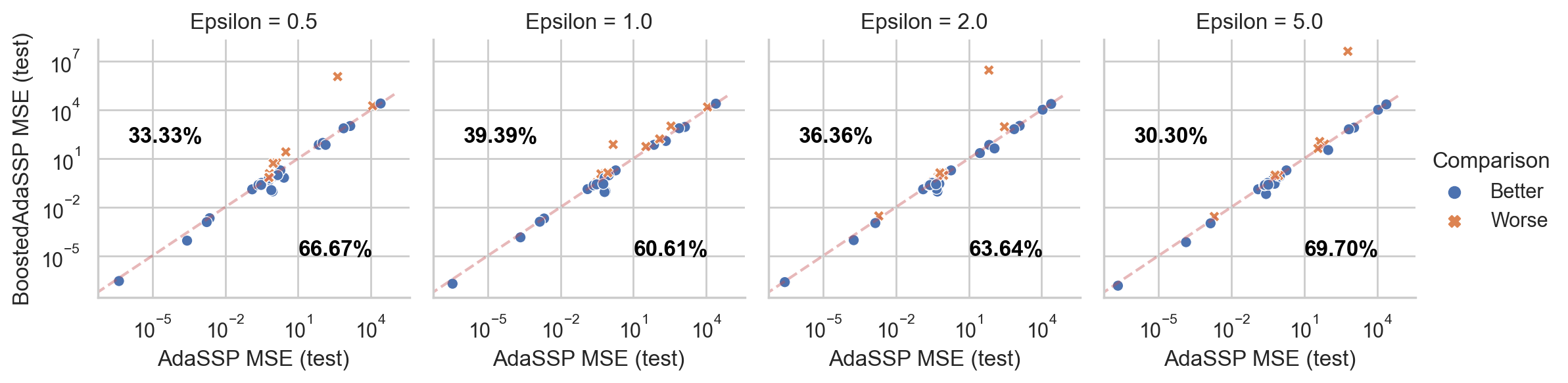}
        \caption{Non-privately Tuned BoostedAdaSSP vs. Non-privately Tuned AdaSSP}
        \label{fig:optvsopt_adassp}
    \end{subfigure}
    \begin{subfigure}[b]{\textwidth}
        \includegraphics[width=\linewidth]{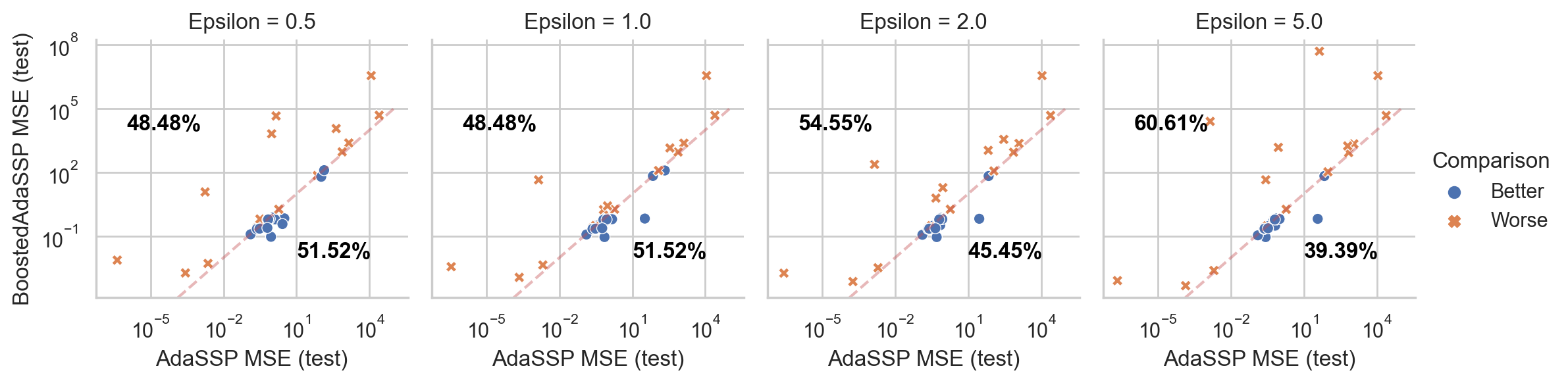}
        \caption{Fixed BoostedAdaSSP vs. Non-privately Tuned AdaSSP}
        \label{fig:fixedvsopt_adassp}
    \end{subfigure}
    \begin{subfigure}[b]{\textwidth}
        \includegraphics[width=\linewidth]{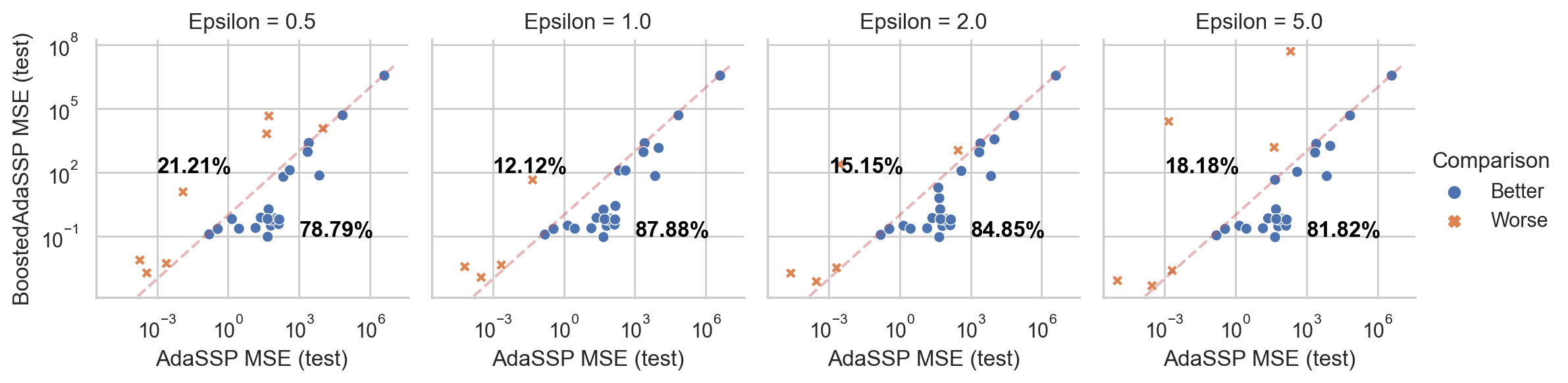}
        \caption{Fixed BoostedAdaSSP vs. Fixed AdaSSP}
        \label{fig:fixedvsfixed_adassp}
    \end{subfigure}    
    \caption{
    \textbf{BoostedAdaSSP vs. AdaSSP. }
    ``Non-privately Tuned'' indicates that hyperparameters of the algorithm are non-privately optimized on each dataset, and ``Fixed'' indicates that the hyperparameters are fixed and shared across all datasets. Each dataset is shown as a point on the plot, labeled with the error obtained by BoostedAdaSSP (y axis) and AdaSSP (x axis). Points below the diagonal are datasets on which BoostedAdaSSP improves over AdaSSP---the fractions of datasets lying above and below the diagonal are annotated. 
    }
    \label{fig:adassp}
\end{figure}

\textbf{BoostedAdaSSP outperforms DP-Gradient Descent}. Gradient descient and BoostedAdaSSP are similar iterative algorithms. But in all comparison settings (including when the hyper-parameters of gradient descent are non-privately optimized on individual datasets, and BoostedAdaSSP uses fixed hyperparameters across all datasets), BoostedAdaSSP substantially outperforms. BoostedAdaSSP can be viewed as gradient descent in function space rather than parameter space, and is able to take advantage of the optimized ridge regression estimator of AdaSSP at each step. 
 Results are in Fig. \ref{fig:dpgd}

\begin{figure}[t]
    \centering
    \begin{subfigure}[b]{\textwidth}
        \includegraphics[width=\linewidth]{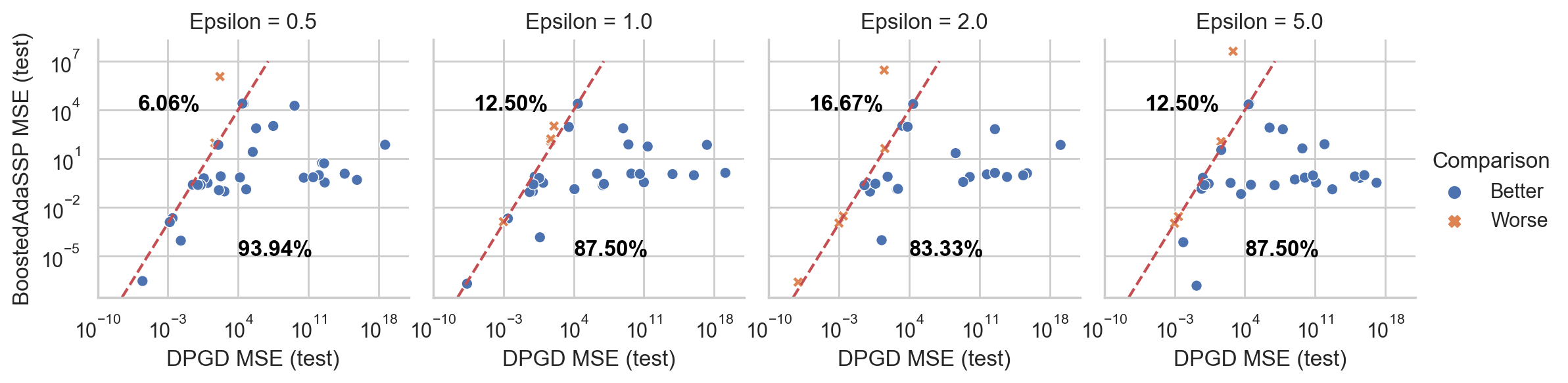}
        \caption{Non-privately Tuned BoostedAdaSSP vs. Non-privately Tuned DP Gradient Descent}
        \label{fig:optvsopt_dpgd}
    \end{subfigure}
    \begin{subfigure}[b]{\textwidth}
        \includegraphics[width=\linewidth]{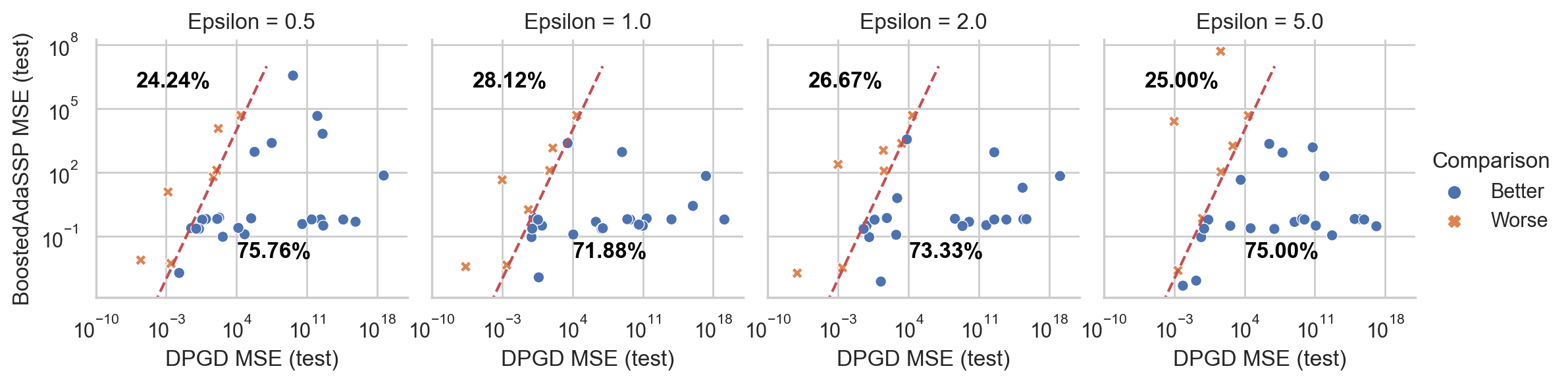}
        \caption{Fixed BoostedAdaSSP vs. Non-privately Tuned DP Gradient Descent}
        \label{fig:fixedvsopt_dpgd}
    \end{subfigure}    
    \begin{subfigure}[b]{\textwidth}
        \includegraphics[width=\linewidth]{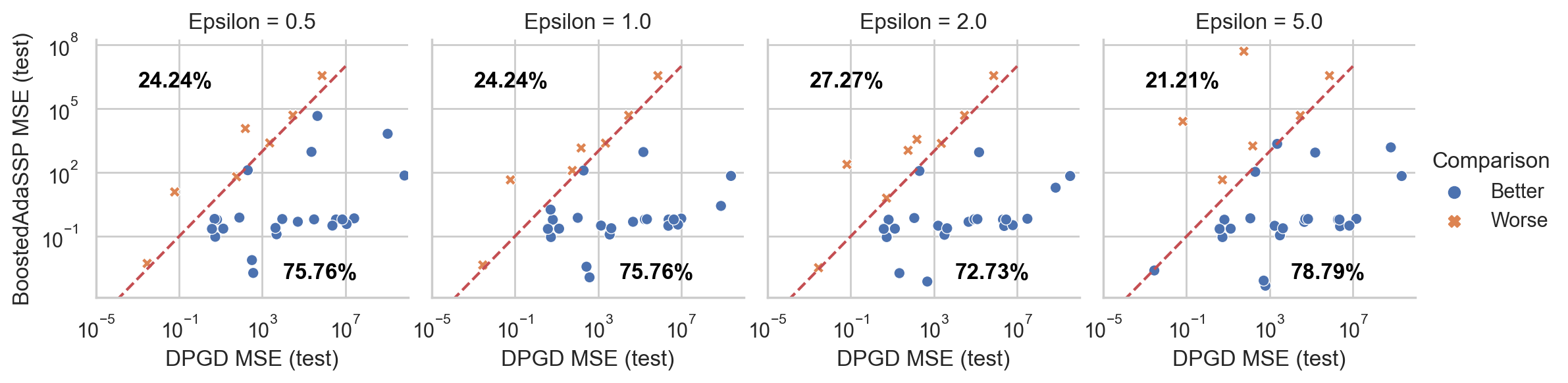}
        \caption{Fixed BoostedAdaSSP vs. Fixed DP Gradient Descent}
        \label{fig:fixedvsfixed_dpgd}
    \end{subfigure}    
    \caption{
    \textbf{BoostedAdaSSP vs. DP Gradient Descent}.
     BoostedAdaSSP outperforms DP Gradient Descent in all comparisons, even when our algorithm uses a fixed set of hyperparameters. 
    }
    \label{fig:dpgd}    
\end{figure}

\textbf{BoostedAdaSSP outperforms TukeyEM}.
BoostedAdaSSP also outperforms TukeyEM in all experimental regimes; we can see that the advantage that BoostedAdaSSP enjoys diminishes as the privacy parameter increases, since (when we optimize for the hyperparameters for both methods), both approach non-private (exact) linear regression. 
 TukeyEM has only one hyperparameter, but it requires a massive number of data samples to train, due to its subsample-and-aggregate nature, and it produces an all-zero parameter vector in many scenarios. In contrast, our BoostedAdaSSP has only a couple more hyperparameters, and a common selection for them works well on many datasets. Comparisons are shown in Fig. \ref{fig:tukeyem}.
\begin{figure}[t]
    \centering
    \begin{subfigure}[b]{\textwidth}
        \includegraphics[width=\linewidth]{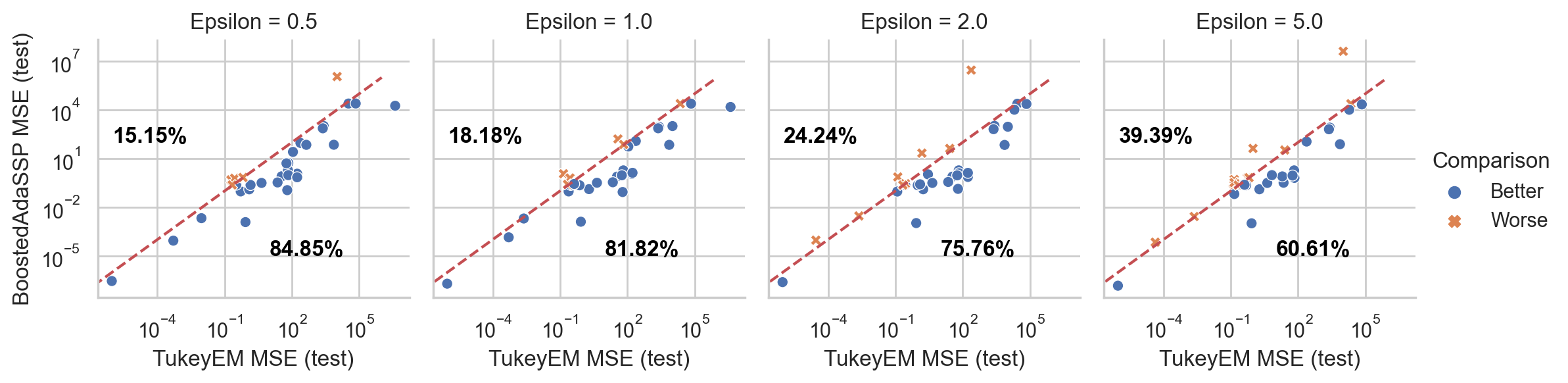}
        \caption{Non-privately Tuned BoostedAdaSSP vs. Non-privately Tuned TukeyEM}
        \label{fig:optvsopt_tukeyem}
    \end{subfigure}
    \begin{subfigure}[b]{\textwidth}
        \includegraphics[width=\linewidth]{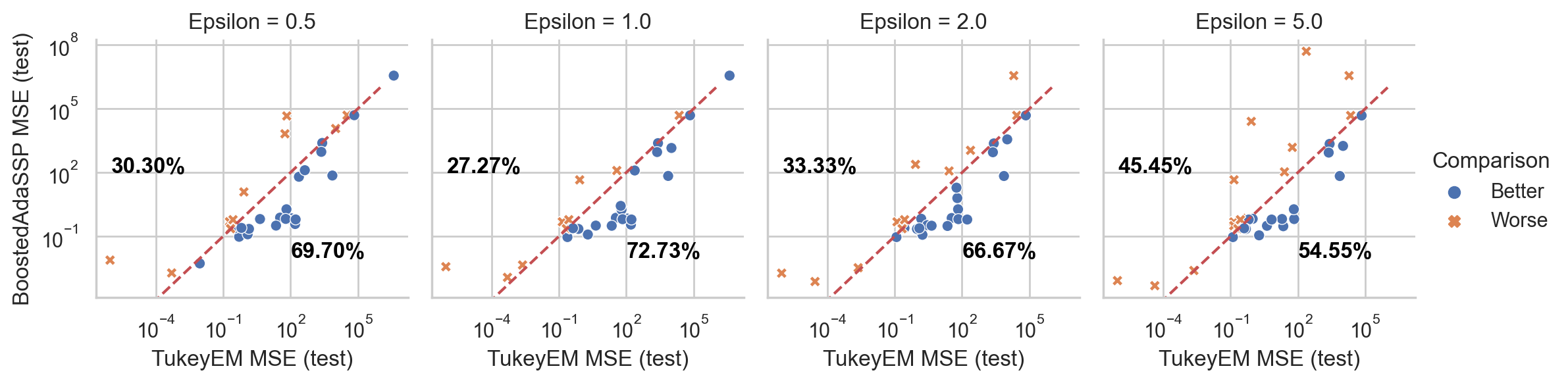}
        \caption{Fixed BoostedAdaSSP vs. Non-privately Tuned TukeyEM}
        \label{fig:fixedvsopt_tukeyem}
    \end{subfigure}
    \begin{subfigure}[b]{\textwidth}
        \includegraphics[width=\linewidth]{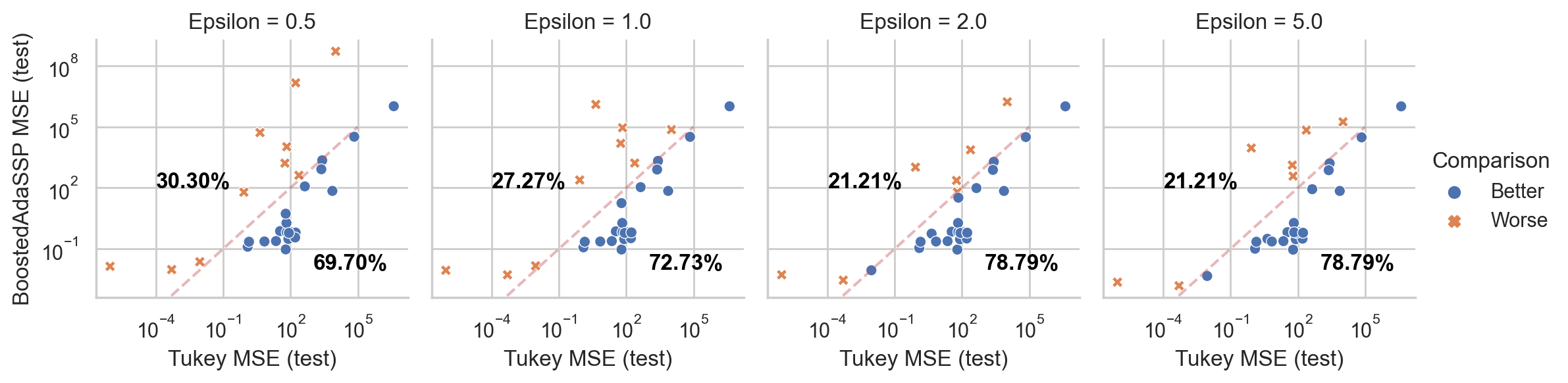}
        \caption{Fixed BoostedAdaSSP vs. Fixed TukeyEM}
        \label{fig:fixedvsfixed_tukeyem}
    \end{subfigure}    
    \caption{
    \textbf{BoostedAdaSSP vs. TukeyEM.}
     BoostedAdaSSP outperforms TukeyEM in all comparisons, even when our algorithm uses a fixed set of hyperparameters.  TukeyEM has the advantage of only having a single hyperparameter (number of models), however, in our experiments we find that there isn't a universally good selection for this hyperparameter. 
    }
    \label{fig:tukeyem}
\end{figure}

\textbf{With privacy, gradient boosting over linear models outperforms gradient boosting over tree based models.} Results in Fig. \ref{fig.DP-EBM} show that BoostedAdaSSP outperforms DB-EBM in all experimental regimes. DB-EBM is also a private gradient boosting algorithm, using tree based learners as base models. This is something that does not occur absent privacy (gradient boosting cannot improve on exact linear regression, as the update steps preserve linearity). This is emblematic of a more general message, that differential privacy rewards algorithmic simplicity (even when more complex algorithms outperform absent privacy constraints). This is because more complex algorithms require more noise addition for privacy, which is often ultimately not worth the tradeoff. 
\begin{figure}[t]
    \centering
    \begin{subfigure}[b]{\textwidth}
        \includegraphics[width=\linewidth]{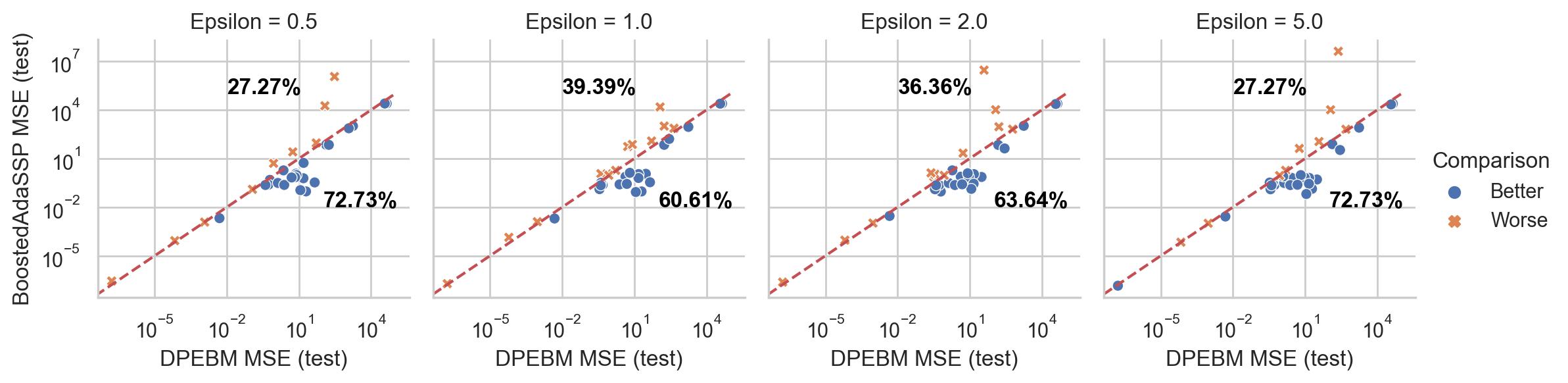}
        \caption{Non-privately Tuned BoostedAdaSSP vs. Non-privately Tuned DP-EBM}
        \label{fig:optvsopt_dpebm}
    \end{subfigure}
    \begin{subfigure}[b]{\textwidth}
        \includegraphics[width=\linewidth]{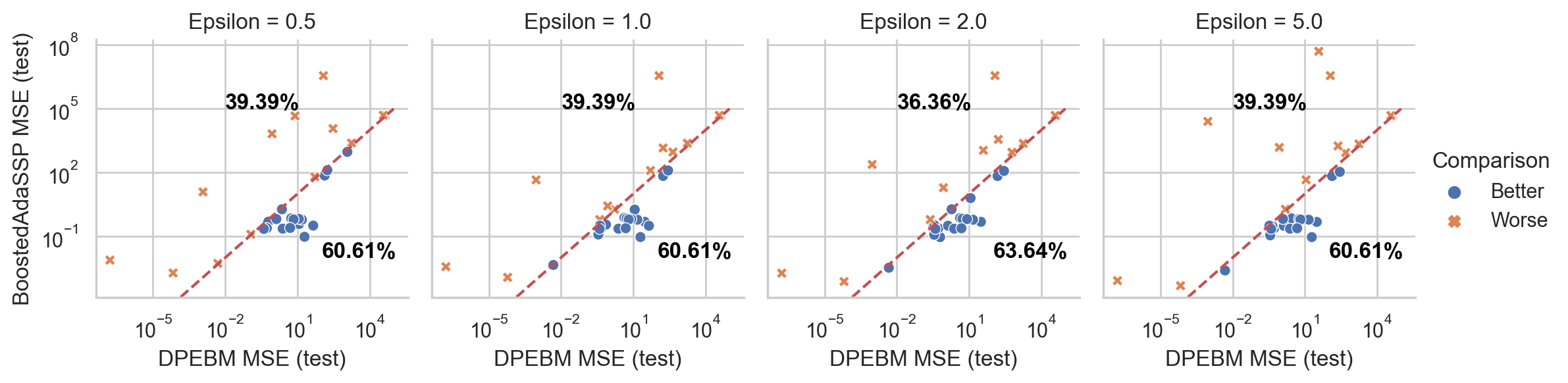}
        \caption{Fixed BoostedAdaSSP vs. Non-privately Tuned DP-EBM}
        \label{fig:fixedvsopt_dpebm}
    \end{subfigure}    
    \begin{subfigure}[b]{\textwidth}
        \includegraphics[width=\linewidth]{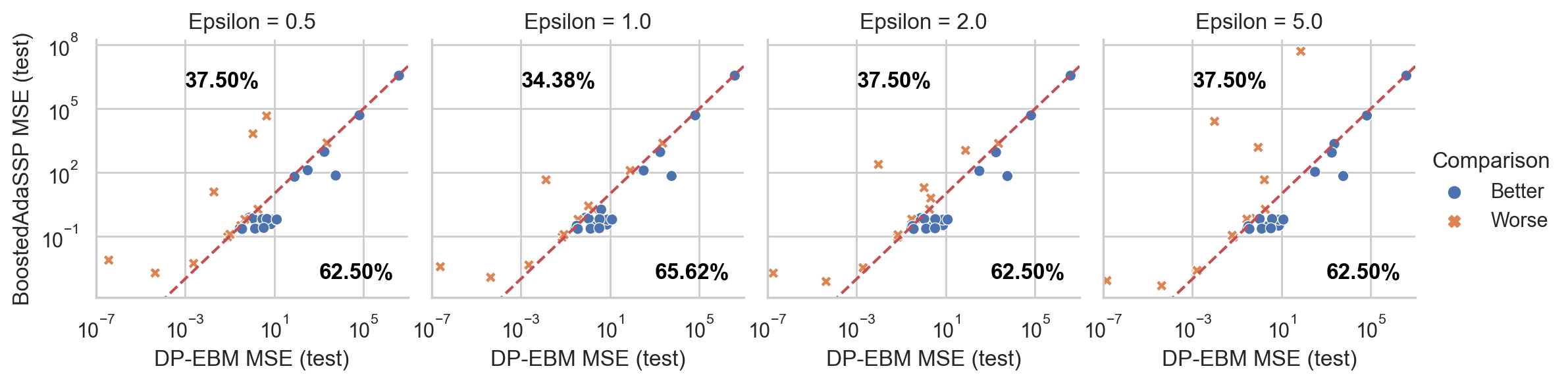}
        \caption{Fixed BoostedAdaSSP vs. Fixed DP-EBM}
        \label{fig:fixedvsfixed_dpebm}
    \end{subfigure}        
    \caption{
    \textbf{BoostedAdaSSP vs. DP-EBM.}
    BoostedAdaSSP and DP-EBM are both gradient boosting algorithms. BoostedAdaSSP uses linear models as the base class, whereas DP-EBM uses tree based models. Our method outperforms in all experimental regimes.
    }
    \label{fig.DP-EBM}
\end{figure}

\section{Theoretical Analysis}
\label{ss.clipping theory}
The improvement of BoostedAdaSSP over the base learner AdaSSP, from a theoretical perspective, can be attributed to the former's ability to adapt to arbitrary data clipping bounds. While the base learner AdaSSP is known to be optimal when the data clipping bounds are data-dependent and well-chosen (\cite{wang2018revisiting}, Theorem 3), it suffers from significant bias when the data clipping bounds are mis-specified (i.e. much closer to 0 relative to the data range). 

This bias exists even in the simplest ``zero-dimensional'' case where linear regression reduces to estimating the population mean of real-valued data. Consider a data set $Y_1, Y_2, \cdots, Y_n \stackrel{i.i.d.}{\sim} \mathcal{N}(\mu, 1)$. With $C_\tau(a) = a\min(1, |a|/\tau)$ denoting the clipping operator, the zero-dimensional AdaSSP estimator is simply $\hat \mu_1 = n^{-1}\sum_{i \in [n]} C_\tau(Y_i) + Z$, where $Z$ is the requisite Gaussian noise for differential privacy. The bias of the AdaSSP estimator, $|\E \hat\mu_1 - \mu|$, is then at least $|\mu| - \tau$, since $|\E \hat \mu_1| \leq \tau$.

In contrast, the BoostedAdaSSP algorithm converges to the population mean $\mu$ for any non-zero clipping bound $\tau$. The zero-dimensional BoostedAdaSSP algorithm for estimating $\mu$ from $Y_1, Y_2, \cdots, Y_n \stackrel{i.i.d.}{\sim} \mathcal{N}(\mu, 1)$ is defined in Algorithm \ref{finite sample meanboosting}.
\begin{theorem}\label{thm:boostedadassp upper bound}
    For every $\tau  = O(1)$ not depending on sample size $n$, Algorithm \ref{finite sample meanboosting} is Gaussian DP with parameter $\rho$, and there exists a data-independent choice of number of boosting rounds $R$ such that the estimator $\hat\mu_R$ converges to the true parameter $\mu$, with the rate of convergence
    \begin{align}
        \E|\hat\mu_R - \mu| = O\left(\frac{\log n}{\sqrt{n}} + \frac{\log^{3/2} n}{n\sqrt{\rho}}\right).
    \end{align}
\end{theorem}
\begin{proof}[Proof of Theorem \ref{thm:boostedadassp upper bound}]
    The Gaussian DP of Algorithm \ref{finite sample meanboosting} follows from the Gaussian mechanism \ref{t.gm} and the composition theorem \ref{c.composition}, by observing that the sensitivity of the clipped sample mean is $2\tau/n$. Next, we establish the convergence of $\hat\mu_R$ by comparing with Algorithm \ref{infinite sample meanboosting}, an ``infinite-sample'' version of Algorithm \ref{finite sample meanboosting}.
\end{proof}

\begin{minipage}{0.475\textwidth}
\begin{algorithm}[H]
		\caption{Zero-dimensional BoostedAdaSSP}
		\label{finite sample meanboosting}
		\begin{algorithmic}
			\Require Clipping function $C_\tau$, number of rounds $R$, Gaussian DP parameter $\rho$.
			\Require Data $Y_1, Y_2, \cdots, Y_n \stackrel{\text{i.i.d.}}{\sim} \mathcal{N}(\mu, 1)$.
			
			\State Initialize: $\hat\mu_0 = 0$
			
			\For{$t \in [R]$}
			\State Compute DP residual mean
			\begin{align}\label{eq: finite-sample iterate}
			    \hat\mu_j &= \hat\mu_{j-1} + \frac{1}{n}\sum_{i \in [n]}C_\tau(Y_i - \hat\mu_{j-1}) + Z_j, \end{align}
       \text{where } $Z_j \sim \mathcal{N}\left(0, \frac{4R\tau^2}{n^2\rho}\right)$.
			
			\EndFor
			\State Output $\hat\mu_R$

		\end{algorithmic}
	\end{algorithm}
\end{minipage}
\hfill
\begin{minipage}{0.475\textwidth}
 \begin{algorithm}[H]
		\caption{Infinite-sample algorithm}
		\label{infinite sample meanboosting}
		\begin{algorithmic}

			\Require Clipping function $C_\tau$, number of rounds $R$.
			\Require Infinite samples from $\mathcal{N}(\mu, 1)$.
			
			\State Initialize: $\theta_0 = 0$
			
			\For{$t \in [R]$}
			\State Compute true truncated residual mean 
			\begin{align}\label{eq: infinite-sample iterate}
			    \theta_j = \theta_{j-1} + \E_{Y \sim \mathcal{N}(\mu, 1)} C_\tau(Y - \theta_{j-1}).
			\end{align}
			\EndFor
			\State Output $\theta_R$

		\end{algorithmic}
	\end{algorithm}
\end{minipage}

By considering an idealized ``infinite sample'' setting where we have access to true distributional quantities $\{\E C_\tau(Y-\theta_{j-1})\}_{j \in R}$, Algorithm \ref{infinite sample meanboosting} removes all the randomness in the finite-sample Algorithm \ref{finite sample meanboosting} and allows us to focus entirely on the bias-reduction effect of boosting. Indeed, the infinite-sample ``estimator'' $\theta_R$ converges deterministically to $\mu$. 

\begin{proposition}\label{prop: infinite sample convergence}
    Suppose the number of rounds $R > \frac{\max(0, |\mu| - \tau)}{(\Phi(2\tau) - 1/2)\tau}$. The error of $\theta_R$ is bounded by
    \begin{align}
        |\theta_R - \mu| \leq \tau \left(3/2 - \Phi(\tau)\right)^{R - \frac{\max(0, |\mu| - \tau)}{(\Phi(2\tau) - 1/2)\tau}}.
    \end{align}
\end{proposition}
That is, after a warm-up of $\frac{\max(0, |\mu| - \tau)}{(\Phi(2\tau) - 1/2)\tau}$ rounds, the error of $\theta_R$ decays geometrically fast, as $0 < 3/2 - \Phi(\tau) < 1$ for any $\tau > 0$. It now suffices to bound the difference $|\theta_R - \hat\mu_R|$.

\begin{proposition}\label{prop: finite sample convergence}
    The difference between outputs of Algorithms \ref{finite sample meanboosting} and \ref{infinite sample meanboosting} is bounded by
    \begin{align}
        \E|\hat\mu_R - \theta_R| = O\left(\frac{R\tau}{\sqrt{n}} + \frac{R^{3/2}\tau}{n\sqrt{\rho}}\right).
    \end{align}
\end{proposition}
By choosing an $R = O(\log n)$ and $R > \frac{\max(0, |\mu| - \tau)}{(\Phi(2\tau) - 1/2)\tau}$, we have $|\theta_R - \mu| = O(\tau/n)$ by Proposition \ref{prop: infinite sample convergence}, and then 
\begin{align}
    \E|\hat\mu_R - \mu| \leq |\theta_R - \mu| + \E|\hat\mu_R - \theta_R|  = O\left(\frac{\tau}{n}\right) + O\left(\frac{\tau \log n}{\sqrt{n}} + \frac{\tau\log^{3/2} n}{n\sqrt{\rho}}\right).
\end{align}
As $\tau = O(1)$ by assumption, the main proof is complete. Propositions \ref{prop: infinite sample convergence} and \ref{prop: finite sample convergence} are proved in Section \ref{ss.proofs}.

\clearpage
\bibliography{main}
\bibliographystyle{alpha}

\newpage
\appendix

\onecolumn

\section{Appendix}
\subsection{Hyperparameters}
\begin{table}[h]
\caption{Hyperparameters of individual algorithms, and their individual tuning ranges. }
\label{tab:hyperparams}
\begin{tabular}{|cccccc|}
\hline
\multicolumn{1}{|c|}{Hyperparameters}     & \multicolumn{1}{c|}{\begin{tabular}[c]{@{}c@{}}Data \\ Clipping\\ Bound\end{tabular}} & \multicolumn{1}{c|}{\begin{tabular}[c]{@{}c@{}}Gradient \\ Clipping\\ Bound\end{tabular}} & \multicolumn{1}{c|}{\begin{tabular}[c]{@{}c@{}}\# Models\\ /Iterations\\ /Updates\end{tabular}} & \multicolumn{1}{c|}{\begin{tabular}[c]{@{}c@{}}Learning \\ Rate\\ Schedule\end{tabular}} & \# Leaves    \\ \hline
\multicolumn{1}{|c|}{Tuning Range}        & \multicolumn{1}{c|}{{[}1e-5,1e+5{]}}                                                  & \multicolumn{1}{c|}{{[}1e-5,1e+5{]}}                                                      & \multicolumn{1}{c|}{{[}1,4000{]}}                                                               & \multicolumn{1}{c|}{\{1,$t^{-1}$,$t^{-1/2}$\}}                                           & {[}1,1000{]} \\ \hline
\multicolumn{5}{|c|}{Algorithm}                                                                                                                                                                                                                                                                                                                                                                                            &              \\ \hline
\multicolumn{1}{|c|}{BoostedAdaSSP}       & \multicolumn{1}{c|}{Yes}                                                              & \multicolumn{1}{c|}{No}                                                                   & \multicolumn{1}{c|}{Yes}                                                                        & \multicolumn{1}{c|}{Yes}                                                                 & No           \\ \hline
\multicolumn{1}{|c|}{AdaSSP}              & \multicolumn{1}{c|}{Yes}                                                              & \multicolumn{1}{c|}{No}                                                                   & \multicolumn{1}{c|}{No}                                                                         & \multicolumn{1}{c|}{No}                                                                  & No           \\ \hline
\multicolumn{1}{|c|}{DP-Gradient Descent} & \multicolumn{1}{c|}{No}                                                               & \multicolumn{1}{c|}{Yes}                                                                  & \multicolumn{1}{c|}{Yes}                                                                        & \multicolumn{1}{c|}{Yes}                                                                 & No           \\ \hline
\multicolumn{1}{|c|}{TukeyEM}             & \multicolumn{1}{c|}{No}                                                               & \multicolumn{1}{c|}{No}                                                                   & \multicolumn{1}{c|}{Yes}                                                                        & \multicolumn{1}{c|}{No}                                                                  & No           \\ \hline
\multicolumn{1}{|c|}{DP-EBM}              & \multicolumn{1}{c|}{Yes}                                                              & \multicolumn{1}{c|}{No}                                                                   & \multicolumn{1}{c|}{Yes}                                                                        & \multicolumn{1}{c|}{No}                                                                  & Yes          \\ \hline
\multicolumn{6}{|c|}{Fixed Hyperparameters}                                                                                                                                                                                                                                                                                                                                                                                               \\ \hline
\multicolumn{1}{|c|}{BoostedAdaSSP}       & \multicolumn{1}{c|}{1}                                                                & \multicolumn{1}{c|}{-}                                                                    & \multicolumn{1}{c|}{100}                                                                        & \multicolumn{1}{c|}{1}                                                                   & -            \\ \hline
\end{tabular}

\end{table}

\subsection{AdaSSP Algorithm for Learning a Single Ridge Regressor}\label{a.adasspalgo}
Let $\;\widehat{.}\;$ denote private versions of the corresponding statistics. Then, AdaSSP privately releases the sufficient statistics of ridge regressor as follows.

\begin{algorithm}[h]
\caption{Private Ridge regression via AdaSSP(data $X, y$, calibration ratio $a, b, c$ , Privacy parameter $\epsilon, \delta$, Bound on $||\mathcal{X}||,||\mathcal{Y}||$)}
\label{a.adassp}
\begin{algorithmic}
\State Find $\mu$ such that $\mu$-GDP satisfies  $(\epsilon, \delta)$-DP. \texttt{ \# Corollary \ref{t.conversion}} 
\State Calibrate $\mu_1, \mu_2, \mu_3$ such that $\mu_1: \mu_2: \mu_3 = a:b:c$ and
$\mu = \sqrt{\mu_1^2 + \mu_2^2 + \mu_3^2}$.
\State Clip $X$ so that $\max_i||x_i||\leq ||\mathcal{X}||$
\State Clip $y$ so that $\max_i||y_i||\leq ||\mathcal{Y}||$
\State $\widehat{X^\top X} = GM(X^\top X, \mu_1, ||\mathcal{X}||)$
\State $\widehat{X^\top y} = GM(X^\top y, \mu_2, \sqrt{||\mathcal{X}||||\mathcal{Y}||})$
\State $\widehat{\lambda} = GM(\lambda_{\min} (X^\top X), \mu_3, ||\mathcal{X}||)$
\State Output $\widehat{\theta_t} =(\widehat{X^\top X} +\widehat{\lambda} I)^{-1} \widehat{X^\top y}$
\end{algorithmic}
\end{algorithm}

Algorithm \ref{a.adassp} instantiates three Gaussian mechanisms with $\mu_1, \mu_2,$ and $\mu_3$ to privately release each sufficient statistic.
Hence the composition 
\begin{align}
    \widehat{\theta_t} =(\widehat{X^\top X} +\widehat{\lambda} I)^{-1} \widehat{X^\top g_t}
\end{align}
is $(\epsilon, \delta)$-DP. Detailed proof is available in Theorem 3 of \cite{wang2018revisiting}.

\clearpage
\onecolumn
\subsection{Datasets}\label{app.data}
All 33 datasets in our experiments come from OpenML. Task information is listed in Tab. \ref{app.regression_tasks}. 

\begin{table}[h]
\centering
\caption{\textbf{Regression Tasks.} `Task ID' refers to the task id on OpenML, $n$ refers to the number of total data points, $d$ refers to the dimension of features after one-hot encoding the categorical features, and `True Bound' indicates the maximum absolute value of labels.}
\label{app.regression_tasks}

\begin{tabular}{|r|r|r|r|r|r|}
\hline
\textbf{Task ID} & \textbf{n} & \textbf{\# Columns} & \textbf{d} & \textbf{True Bound} & \textbf{n / d} \\ \hline
361072           & 8192       & 21                  & 21         & 99.000000           & 390.095238                \\ \hline
361073           & 15000      & 26                  & 26         & 100.000000          & 576.923077                \\ \hline
361074           & 16599      & 16                  & 16         & 0.078000            & 1037.437500               \\ \hline
361075           & 7797       & 613                 & 613        & 26.000000           & 12.719413                 \\ \hline
361076           & 6497       & 11                  & 11         & 9.000000            & 590.636364                \\ \hline
361077           & 13750      & 33                  & 33         & 0.003600            & 416.666667                \\ \hline
361078           & 20640      & 8                   & 8          & 13.122367           & 2580.000000               \\ \hline
361079           & 22784      & 16                  & 16         & 13.122367           & 1424.000000               \\ \hline
361080           & 53940      & 6                   & 6          & 9.842888            & 8990.000000               \\ \hline
361081           & 10692      & 8                   & 8          & 13.928840           & 1336.500000               \\ \hline
361082           & 17379      & 6                   & 6          & 977.000000          & 2896.500000               \\ \hline
361083           & 581835     & 9                   & 9          & 5.528238            & 64648.333333              \\ \hline
361084           & 21613      & 15                  & 15         & 15.856731           & 1440.866667               \\ \hline
361085           & 10081      & 6                   & 6          & 1.000000            & 1680.166667               \\ \hline
361086           & 163065     & 3                   & 3          & 11.958631           & 54355.000000              \\ \hline
361087           & 13932      & 13                  & 13         & 14.790071           & 1071.692308               \\ \hline
361088           & 21263      & 79                  & 79         & 185.000000          & 269.151899                \\ \hline
361089           & 20640      & 8                   & 8          & 1.791761            & 2580.000000               \\ \hline
361090           & 18063      & 5                   & 5          & 12.765691           & 3612.600000               \\ \hline
361091           & 515345     & 90                  & 90         & 2011.000000         & 5726.055556               \\ \hline
361092           & 8885       & 62                  & 82         & 1.000000            & 108.353659                \\ \hline
361093           & 4052       & 7                   & 12         & 2.300000            & 337.666667                \\ \hline
361094           & 8641       & 4                   & 5          & 40.000000           & 1728.200000               \\ \hline
361095           & 166821     & 9                   & 23         & 10.084141           & 7253.086957               \\ \hline
361096           & 53940      & 9                   & 26         & 9.842888            & 2074.615385               \\ \hline
361097           & 4209       & 359                 & 735        & 265.320000          & 5.726531                  \\ \hline
361098           & 10692      & 11                  & 17         & 13.928840           & 628.941176                \\ \hline
361099           & 17379      & 11                  & 20         & 977.000000          & 868.950000                \\ \hline
361100           & 39644      & 59                  & 73         & 13.645079           & 543.068493                \\ \hline
361101           & 581835     & 16                  & 31         & 5.528238            & 18768.870968              \\ \hline
361102           & 21613      & 17                  & 19         & 15.856731           & 1137.526316               \\ \hline
361103           & 394299     & 6                   & 26         & 6.480505            & 15165.346154              \\ \hline
361104           & 241600     & 9                   & 15         & 8.113915            & 16106.666667              \\ \hline
\end{tabular}
\end{table}

\clearpage

\subsection{Omitted Proofs in Section \ref{ss.clipping theory}} \label{ss.proofs}

\subsubsection{Proof of Proposition \ref{prop: infinite sample convergence}}

Let $\Delta_t = \theta_{t-1} - \mu$, so that $|\Delta_t|$ is the bias of Algorithm \ref{infinite sample meanboosting} after $t-1$ iterations. We would like to see $|\Delta_t|$ decay as quickly as possible in $t$. The following lemmas quantify the rate of decay.

\begin{lemma}\label{lem:smaller}
Let $t \geq 0$ and $\tau$ be the clipping threshold. Suppose $|\Delta_t| \leq \tau$. Then 
$
    | \Delta_{t+1} | \leq \left(\frac{3}{2} - \Phi(\tau)\right) |\Delta_{t}|,
$
where $\Phi(\cdot)$ is the standard Gaussian CDF.
\end{lemma}

\begin{lemma}\label{lem:bigger}
Let $t \geq 0$ and $\tau$ be the clipping threshold. Suppose $|\Delta_{t}| > \tau$. Then
$|\Delta_{t+1}| \leq |\Delta_{t}| - (\Phi(2\tau) - 1/2)\tau,$
where $\Phi(\cdot)$ is the standard Gaussian CDF.
\end{lemma}

The lemmas taken together suggest that, if $|\mu| > \tau$, then it takes 
$\frac{|\mu| - \tau}{(\Phi(2\tau) - 1/2)\tau}$ rounds for the error $|\Delta_{t}|$ to decrease below $\tau$. As soon as $|\Delta_{t}| \leq \tau$, 
then $|\Delta_{t}|$ decays geometrically by a factor of $(3/2 - \Phi(\tau))$ each round. Then, for $R > \frac{\max(0, |\mu| - \tau)}{(\Phi(2\tau) - 1/2)\tau}$, the desired bound in Proposition \ref{prop: infinite sample convergence} follows.

It remains to prove the two lemmas.
\begin{proof}[Proof of Lemma \ref{lem:smaller}]
Let $m_t$ denote the true truncated residual mean in the $t$-th step. 

Throughout the proof, let $P_t$ denote $\mathcal N(\Delta_t, 1)$. Note that if $\mu_{t} = 0$, then $m_t= \Delta_t=\Delta_{t+1}$, and so the inequality holds. 

Now we consider the case where $\Delta_t \neq 0$. Without loss of generality, assume $\mu_{t} > 0$. Let $C$ denote the clipping operation, that is $C(Y) = Y \min(1, \tau /|Y|)$.
Then we can decompose the estimate $m_t$ as follows:
\begin{align}
m_t &= \E_{P_t}[C(Y)]\\
 &= \underbrace{\Pr_{P_t}[Y < -\tau] (-\tau) + \Pr_{P_t}[Y > \tau + 2\Delta_t] \tau}_{T_1}\nonumber\\ 
&+ \underbrace{\Pr_{P_t}[Y\in [-\tau, \tau+2\Delta_t]]  \E[C(Y) \mid Y\in [-\tau, \tau+2\Delta_t]]}_{T_2}
\end{align}
Since the distribution $P_t$ is symmetric about $\Delta_t$, $T_1=0$. 
Now we further decompose $T_2$ by considering three different intervals:
\begin{align}
    T_2 &= \underbrace{\Pr_{P_t}[Y \in [2\Delta_t - \tau, \tau]] \E_{P_t}[C(Y) \mid Y\in [2\Delta_t - \tau, \tau]]}_{T_3}\nonumber\\
    &+ \underbrace{\Pr_{P_t}[Y \in [- \tau, 2\Delta_t - \tau]] \E_{P_t}[C(Y) \mid Y\in [- \tau, 2\Delta_t - \tau]]}_{T_4}\nonumber\\
    &+ \underbrace{\Pr_{P_t}[Y \in [\tau, \tau + 2\Delta_t]] \E_{P_t}[C(Y) \mid Y\in [\tau, \tau + 2\Delta_t]]}_{T_5}
\end{align}

In the interval of $T_3$, no $Y$ is clipped, so $C(Y) = Y$. Since the interval is also centered at the mean $\Delta_t$, the conditional expectation is $\Delta_t$, so
\begin{equation}
\label{T3} T_3 = \Pr_{P_t}[Y \in [2\Delta_t - \tau, \tau]] \Delta_t.    
\end{equation}

In the interval of $T_5$, each $Y$ is clipped, so
\begin{equation}
\label{T5} T_5 = \Pr_{P_t}[Y \in [\tau, \tau + 2\Delta_t]] \tau.    
\end{equation}

In the interval of $T_4$, no $Y$ is clipped. Moreover, for any $y, y' \in [-\tau, 2\Delta_t - \tau]$ such that $y > y'$, the density $P_t(y) > P_t(y')$. This allows us to lower bound $T_4$:
\begin{align}
\notag    T_4 &= \int_{[-\tau , 2\Delta_t - \tau]} Y \, P_t(Y) \\
\notag &=  \int_{[\Delta_t - \tau, 2\Delta_t - \tau]} Y \, P_t(Y) + \int_{[-\tau , \Delta_t - \tau]} Y \, P_t(Y) \\
\notag    &= \int_{[\Delta_t - \tau , 2\Delta_t - \tau]} \left[Y\, P_t(Y) + (2 \Delta_t - 2\tau - Y) P_t(2 \Delta_t - 2\tau - Y)\right]\\
\label{abcd}    &\geq \int_{[\Delta_t - \tau , 2\Delta_t - \tau]}
(\Delta_t - \tau) \, (P_t(Y) + P_t(2 \Delta_t - 2\tau - Y))\\
\notag &= (\Delta_t - \tau) \int_{[-\tau , 2\Delta_t - \tau]} P_t(Y) \\
\label{T4} &= (\Delta_t - \tau) \Pr_{P_t}[Y\in [-\tau , 2\Delta_t - \tau]]
\end{align}
where the step in inequality \eqref{abcd} follows from the fact that for any four numbers $a, b, c, d$ such that $a\geq c$ and $b\geq d$, then $ab  + cd \geq \frac{(a+c)}{2}(b+d)$. Finally, note that $\Pr_{P_t}[Y\in [-\tau , 2\Delta_t - \tau]] = \Pr_{P_t}[Y\in [\tau , 2\Delta_t + \tau]]$ since the two intervals are symmetric about $\Delta_t$. Thus,
\begin{align}
    T_4 + T_5 = \frac{\Delta_t}{2} \left( \Pr_{P_t}[Y\in [-\tau , 2\Delta_t - \tau]] + \Pr_{P_t}[Y\in [\tau , 2\Delta_t + \tau]]\right)
\end{align}

Putting \eqref{T3}, \eqref{T5} and \eqref{T4} together, we get
\begin{align}\label{hat mu_t lower bound}
m_t= T_3 + T_4 + T_5 \geq \frac{\Delta_t}{2} \Pr_{P_t}[Y \in [-\tau , \tau + 2\Delta_t]]
\end{align}

Finally, note that
\begin{align}
\Pr_{P_t}[Y \in [-\tau , \tau + 2\Delta_t]] &= \Pr_{P_t}[Y - \Delta_t \in [-\tau - \Delta_t , \tau + \Delta_t]] \\ 
&\geq \Pr_{P_t}[Y - \Delta_t \in [-\tau , \tau]] \\ &= 2 \Phi(\tau) - 1
\end{align}
This means
\begin{align}
\Delta_{t+1} = \Delta_t - m_t\leq (3/2 - \Phi(\tau)) \Delta_t
\end{align}
which completes the proof.
\end{proof}

\begin{proof}[Proof of Lemma \ref{lem:bigger}]
Throughout the proof, let $P_t$ denote $\mathcal N(\Delta_t, 1)$. Without the loss of generality, assume $\Delta_t > 0$. We start by decomposing the mean of the clipped distribution:
\begin{align}
m_t &= \E_{P_t}[C(Y)]\\
 &= \underbrace{\Pr_{P_t}[Y < -\tau] (-\tau) + \Pr_{P_t}[Y > \tau + 2\Delta_t] \tau}_{T_1}\nonumber\\
&+ \underbrace{\Pr_{P_t}[Y\in [-\tau, \tau+2\Delta_t]]  \E[C(Y) \mid Y\in [-\tau, \tau+2\Delta_t]]}_{T_2}
\end{align}
By the symmetric property of $P_t$, $T_1 = 0$. Now we will further break down $T_2$ into the parts that were unclipped and clipped:
\begin{align}
\notag    T_2 &= \underbrace{\Pr_{P_t}[Y\in [-\tau, \tau]]  \E[Y\mid Y\in [-\tau, \tau]]}_{T_3}\\
    &+ \underbrace{\Pr_{P_t}[Y\in [\tau, 2\Delta_t + \tau]] \tau}_{T_4}
\end{align}
First, note that for each $y\in [0, \tau]$, we have $P_t(y) \geq P_t(-y)$ since $y$ is closer than $-y$ to the mean $\Delta_t$. This implies
\begin{align}
\notag    T_3 &= \int_{[-\tau, \tau]} Y \, P_t(Y)\\
\notag   &= \int_{[0, \tau]} [Y \, P_t(Y) + (-Y) \, P_t(-Y)]\\
    &\geq \int_{[0, \tau]} [Y \, P_t(Y) + (-Y) \, P_t(Y)] = 0
\end{align}
Finally, we note that in $T_4$, the probability of interval can be lower bounded as:
\begin{align}
\notag    \Pr_{P_t}[Y\in [\tau, 2\Delta_t + \tau]] &= \Pr_{Z\sim \mathcal{N}(0, 1)}[Z \in [\tau - \Delta_t, \Delta_t + \tau]]\\
 &\geq \Pr_{Z\sim \mathcal{N}(0, 1)}[Z \in [0, 2 \tau]]
\end{align}
where the last inequality follows from $\tau \leq \Delta_t$.
Thus, $m_t \geq (\Phi(2\tau) - 1/2) \tau$, which recovers the stated bound.\end{proof}

\subsubsection{Proof of Proposition \ref{prop: finite sample convergence}}
Define
\begin{align} 
    A_{j, n} := \frac{1}{n}\sum_{i \in [n]}C_\tau(Y_i - \theta_j) - \E C_\tau(Y - \theta_j), \quad B_{j, n} := \frac{1}{n}\sum_{i \in [n]} \left(C_\tau(Y_i - \hat\mu_j) - C_\tau(Y_i - \theta_j)\right),
\end{align}
then for every $j \in [R]$ it holds, by equations \eqref{eq: finite-sample iterate} and \eqref{eq: infinite-sample iterate}, that
\begin{align}\label{eq: comparison main expansion}
    \hat\mu_j - \theta_j = \hat\mu_{j-1} - \theta_{j-1} + A_{j-1, n} + B_{j-1, n} + Z_j.
\end{align}
To simplify the right side, observe that: 
\begin{itemize}
    \item each term in $B_{j-1, n}$, $C_\tau(Y_i - \hat\mu_{j-1}) - C_\tau(Y_i - \theta_{j-1})$ is of the same sign as $(Y_i - \hat\mu_{j-1}) - (Y_i - \theta_{j-1}) = \theta_{j-1} - \hat\mu_{j-1}$, since clipping preserves ordering: if $a \leq b$, then $C_\tau(a) \leq C_\tau(b)$;
    \item the magnitude $|C_\tau(Y_i - \hat\mu_{j-1}) - C_\tau(Y_i - \theta_{j-1})|$ is upper bounded by $|(Y_i - \hat\mu_{j-1}) - (Y_i - \theta_{j-1})| = |\theta_{j-1} - \hat\mu_{j-1}|$, as clipping is non-expansive: for any $a, b$, $|C_\tau(a) - C_\tau(b)| \leq |a-b|$.
\end{itemize}
It follows that $|\hat\mu_{j-1} - \theta_{j-1} + B_{j-1, n}| \leq |\hat\mu_{j-1} - \theta_{j-1}|$, and therefore \eqref{eq: comparison main expansion} implies
\begin{align}
     |\hat\mu_j - \theta_j| \leq |\hat\mu_{j-1} - \theta_{j-1}| + |A_{j-1, n}| + |Z_j|.
\end{align}
Since $|\hat\mu_0 - \theta_0| = 0$ by definition, we have
\begin{align*}
    |\hat\mu_R - \theta_R| \leq \sum_{j=1}^{R-1} |A_{j-1, n}| + \sum_{j=1}^{R} |Z_j|.
\end{align*}
The desired bound in Proposition \ref{prop: finite sample convergence} is then the consequence of two observations.
\begin{itemize}
    \item Each $A_{j, n}$ is the difference between the sample mean of i.i.d. bounded random variables $\{C_\tau(Y_i - \theta_j)\}_{i \in [n]}$ and their expectation. $\E|A_{j-1, n}| \leq \sqrt{\E A^2_{j-1, n}} = O\left(\frac{\tau}{\sqrt{n}}\right)$.
    \item The $Z_j$'s are independently drawn from $\mathcal{N}(0, \frac{4R\tau^2}{n^2\rho})$. We have $E |Z_j| = O(\frac{\sqrt{R}\tau}{n\sqrt{\rho}})$.
\end{itemize}
\clearpage

\subsection{Optimality of Boosting under Lossless Clipping}
\label{ss.non clipping theory}
When the true data scales $\|\mathcal X\|, \|\mathcal Y\|$ are known, \cite{wang2018revisiting} studies the rate of convergence of AdaSSP estimator $\hat\theta_{\text{AdaSSP}}$ to the non-private least squares estimator $\theta^*$ by showing that, under mild regularity conditions for the design matrix $X$, the squared distance $\|\hat\theta_{\text{AdaSSP}} - \theta^*\|^2$ is less than $C\frac{\|\mathcal X\|^2\text{tr}[(X^\top X)^{-2}]}{\epsilon^2/\log(1/\delta)}$ with probability at least $1-\delta/3$ (Theorem 3, \cite{wang2018revisiting}) for some constant $C$. As argued in the original paper, this rate of convergence is optimal for $(\epsilon, \delta)$-differentially private linear regression (\cite{bassily2014private, cai2021cost}). 

In Theorem \ref{thm:regression} we show that the boosted algorithm can attain the same rate of convergence, which helps explain why BoostedAdaSSP performs no worse than one-shot AdaSSP in our experiments when the clipping threshold is data-dependent.

\begin{theorem}
\label{thm:regression}
    If $y|X$ is sampled from a Gaussian linear model and the minimum eigenvalue of the Gram matrix satisfies $\lambda_{\min}(X^\top X) \geq \frac{\alpha n\|\mathcal X\|^2}{d}$ for some $\alpha > 0$, then there exists a high-probability event $E$ not depending on $y|X$ such that, as long as the number of boosting rounds $T = O(1)$, the BoostedAdaSSP estimator $\hat{\theta}_T$ satisfies 
    \begin{align}
        \E(\|\hat\theta_T - \theta^*\|^2|X, E) \leq C\frac{\|\mathcal X\|^2\text{tr}[(X^\top X)^{-2}]}{\epsilon^2/\log(1/\delta)}
    \end{align}
    for some constant $C$.
\end{theorem}

\begin{proof}
Let $M$ denote the sum of $\hat \lambda I$ (where $\hat\lambda$ is defined in Algorithm \ref{a.dpgb_linear}) and the symmetric Gaussian matrix perturbation to $X^\top X$. convergence of Algorithmsider two convergence of Algorithmsecutive iterates of Algorithm \ref{a.dpgb_linear}. We have
\begin{align}
    \theta_{t+1} = (X^\top X + M)^{-1}X^\top(y - X\theta_t) + Z_{t+1},
\end{align}
where $Z_{t+1} \sim \mathcal N_d\left(0, \nu^2(X^\top X + M)^{-2}\right)$ is the fresh Gaussian noise drawn in the $(t+1)$-th iteration. The coefficient $\nu^2$ in the covariance matrix depends on privacy parameters and data scales and shall be specified later.

With $\theta^* = (X^\top X)^{-1}X^\top y$, rearranging terms yields
\begin{align}
    \theta_{t+1} - \theta^* = (I - (X^\top X + M)^{-1}X^{\top} X)(\theta_t - \theta^*) + Z_{t+1}.
\end{align}
With the same choice of high-probability event as Section B of \cite{wang2018revisiting}, we have $0.5(X^\top X + \hat\lambda I) \preceq X^\top X + M \preceq 2(X^\top X + \hat\lambda I)$, and therefore there exists some absolute convergence of Algorithmstant $0 < \kappa < 1$ such that
\begin{align}
\|\theta_{t+1} - \theta^*\| \leq (1-\kappa)\|\theta_{t} - \theta^*\| + \|Z_{t+1}\|.
\end{align}
Iterating this noisy convergence of Algorithmtraction yields
\begin{align}
\|\theta_{t+1} - \theta^*\| \leq (1-\kappa)^t\|\theta_{1} - \theta^*\| + \sum_{j=1}^t (1-\kappa)^{t-j}\|Z_{j}\|.
\end{align}
To bound the right side in expectation, observe that $\E( \|\theta_{1} - \theta^*\|^2|X, E)$ is of the same order as $\E\|\hat\theta_{\text{AdaSSP}} - \theta^*\|^2|X, E$ when $T = O(1)$. For the noise term, again with $T = O(1)$, the requisite $\nu^2$ in AdaSSP is of the order $\frac{\|\mathcal X\|\|\mathcal Y\|}{\epsilon^2/\log(1/\delta)}$; the overall magnitude of the noise term $\frac{\|\mathcal X\|^2\|\mathcal Y\|^2}{\epsilon^2/\log(1/\delta)}\text{tr}[(X^\top X + M)^{-2}]$ is bounded by $C\frac{\|\mathcal X\|^2\text{tr}[(X^\top X)^{-2}]}{\epsilon^2/\log(1/\delta)}$ under the high probability event for $(X^\top X + M)^{-1}$ and with the same choice of $C$ in Theorem 2(iii) of \cite{wang2018revisiting}.
\end{proof}
\clearpage

\subsection{Some Additional Theoretical Perspectives}
In this section, we provide some additional theoretical justification of BoostedAdaSSP. 
\begin{itemize}
    \item In Section~\ref{s.separation2}, we will prove a more fine-grained finite sample separation result between one-shot AdaSSP and BoostedAdaSSP: even if we know a bounded support of the data a priori at $[-B,B]$, and even if we can choose the threshold $\tau$ optimally as a function of this bound $B$ and the sample size $n$, BoostedAdaSSP can already adapt to the small-variance of the actual data distribution with $T=2$ steps, while non-BoostedAdaSSP with an optimally chosen $\tau$ cannot do any better than the worst-case that depends on the global boundedness parameter $B$.
    \item In Section~\ref{s.robustness}, we derive a new interpretation of BoostedAdaSSP as an iterative optimization algorithm optimizing a ``robustified'' objective function from a fixed-point iteration perspective. This analysis offers new theoretical insight into the practical benefits of choosing large $T$ and small $\tau$.
\end{itemize}

 Throughout this section, $[x]_{[a, b]} = \min(\max(x, b), a)$ denotes the clipping operator. To avoid the notational collision with the $\mu$-Gaussian Differential Privacy, we use define $\rho := \mu^2$ and will use $\rho$ for the privacy parameter throughout the section. This can be interpreted as $\rho$-zCDP or $\sqrt{\rho}$-GDP.  The symbol $\mu$ will reserved for the mean of the random variable.
\subsubsection{Finite-Sample Separation between AdaSSP and BoostedAdaSSP}\label{s.separation2}
To see the constant error incurred by AdaSSP in the finite-sample setting, consider estimating $\mu = \E Y $ using a private data set $Y_1,...,Y_n\sim \mathcal{P}$. The worst-case MSE of one shot AdaSSP is characterized by the following theorem. 

\begin{theorem}\label{theorem:mean separation}
Suppose there exists parameter $B,\sigma,\mu$, such that the distribution $\mathcal{P}$ of random variable $Y\in \mathbb R$ satisfies that  $\Pr(|Y|\leq B) = 1$,  (b) $\mu= \E Y $, (c) $Y-\mu$ is $\sigma^2$-subgaussian.  Let $\hat\mu^\tau_n(Y)$ be the AdaSSP estimator with clipping at $\tau$. If its zCDP parameter $\rho< n$ and $n \geq C\sqrt{\log n / \rho}$ for a universal constant $C$, then for any clipping level $\tau$ and any $n$, we have
\begin{equation}\label{eq:lowerbound_T=1}
\max_{\mathcal P} \E [(\hat{\mu}^{\tau}(Y) - \bar{\mu})^2]  \geq \frac{ \tau^2 + \mu^2}{18n^2\rho} + \frac{2{\max}^2(B-\tau, 0)}{9} 
\end{equation}
\end{theorem}
In addition,  there exists parameters $\tau_1,\tau_2$ such that BoostedAdaSSP for two iterations with threshold in first round chosen as $\tau_1$ and second round chosen as $\tau_2$ such that 
\begin{equation}\label{eq:upperbound_T=2}
 \E[ (\hat{\mu}_2^{\tau_{1:2}}(Y) - \bar{\mu})^2 ] = O\left(\frac{\sigma^2 \log(n\rho)+ \mu^2}{n^2\rho} + \frac{B^2+\mu^2}{n^4 \rho^2}\right).
\end{equation}
Theorem \ref{theorem:mean separation} implies that the one-step AdaSSP cannot gain by choosing $\tau < B$ without incurring a \emph{non-vanishing} constant asymptotic error. Moreover, for large $n$, a dependence on $B$ in the leading term is necessary no matter how $\tau$ is chosen.

Meanwhile, BoostedAdaSSP with $T=2$ is able to get rid of the dependence in $B$ from the leading term by adaptively choosing the clipping threshold. This separation can be orders-of-magnitude when $B \gg \max\{\sigma,\mu\}$.

We further note that the expression of interest we consider is how well a differentially private estimator can approximate the empirical mean $\bar{\mu}$. Results for estimating the population level mean parameter $\mu$ are directly implied, since $\E[(\bar{\mu} - \mu)^2] \leq \frac{\sigma^2}{n}$. Whenever $\rho$ is small, or $\mu\gg \sigma$ is large, or $B\gg q$, the additional error due to DP from \eqref{eq:lowerbound_T=1} and \eqref{eq:upperbound_T=2} could easily become larger than the statistical error for small to moderate sized data.  From this perspective, Boosting in AdaSSP could make a difference in enabling applications of private data analysis to significantly smaller datasets than its non-boosted counterpart.

\begin{proof}[{Proof of Theorem \ref{theorem:mean separation}}]
	Let the noise added be $Z_1, Z_2$ from the first round of AdaSSP.  Specifically, $Z_1$ is added to $n$,  and $Z_2$ is added to $\sum_i (Y_i)_{[-\tau,\tau]}$. $Z_3$ is added to $\sum_i (Y_i - \hat{\mu}_1)_{[-\tau_2,\tau_2]}$.
	
	Condition on the event $E_1$ that $|Z_1| \leq \sigma_1\sqrt{2\log(2/\delta)}$, which happens with probability $\geq 1-\delta$ by the standard Gaussian tail bound.
	Under the assumption $n > 2\sigma_1\sqrt{2\log(2/\delta)}$ (we will work out the choice of $\sigma_1$ and $\delta$ later such that this will match what's stated in the theorem), this implies that  $|Z_1| \leq n/2$.  Observe that the conditional random variable $Z_1 | E_1$ remains $\sigma^2$-subgaussian with a variance at least $0.5\sigma^2$. \footnote{This can be checked by the variance of truncated Gaussian random variable.}
Also observe that,  under the same event and assumption on $n$, the clipping in the denominator at $1$ does not occur, i.e., $\max\{1, n+Z_1\} = n + Z_1$.

Thus under $E_1$, we can write
	\begin{align}
	\hat{\mu}^\tau(Y) - \bar{\mu} &= \frac{\sum_i[Y_i]_{[-\tau,\tau]} + Z_2}{n + Z_1} -  \frac{\sum_i Y_i }{n } \nonumber\\
	&= \frac{\sum_i [Y_i]_{[-\tau,\tau]}  + Z_2}{n + Z_1} - \frac{\sum_i Y_i}{n + Z_1}+  \frac{\sum_i Y_i}{n + Z_1}-  \frac{\sum_i Y_i }{n } \nonumber\\
	&= \frac{\sum_i[Y_i]_{[-\tau,\tau]} - Y_i }{n+Z_1}  \frac{Z_2}{ n + Z_1} +  \frac{- Z_1 \sum_i Y_i}{n(n + Z_1)} \nonumber\\
	&= \underbrace{\frac{\sum_i[Y_i]_{[-\tau,\tau]} - Y_i }{n+Z_1}}_{(*)}  + \underbrace{\frac{Z_2 - Z_1 \bar{\mu}}{ n + Z_1}}_{(**)}. \label{eq:useful_decomp}
\end{align}

Next, we discuss two cases.  First consider  $\tau \geq  B$, the first term $(*)$ is $0$ and  it remains to bound the second term $(**)$
\begin{align}
\left|\frac{ (Z_2 - Z_1\bar{\mu})}{3n/2}\right|^2 \leq |\hat{\mu}^\tau- \bar{\mu}|^2 \leq \left|\frac{ (Z_2 - Z_1\bar{\mu})}{n/2}\right|^2
\end{align}
where conditioning on $Y_{1:n}$ and $E_1$,  $\frac{ (Z_2 - Z_1\bar{\mu})}{n/2}$ is  subgaussian with parameters $\frac{\sigma_2^2 + \bar{\mu}^2\sigma_1^2}{(n/2)^2}$.
Note that $\sigma_2^2 = \frac{\tau^2}{\rho}$ and $\sigma_1^2 = \frac{1}{\rho}$ (for an even splitting of the privacy budget). Taking conditional expectation gives
\begin{equation}\label{eq:square_err_sandwich}
\frac{4  ( \tau^2 + \bar{\mu}^2/4) }{9n^2\rho} \leq \E[(\hat{\mu}^\tau- \bar{\mu})^2  | \bar{\mu}, E_1] \leq  \frac{4  ( \tau^2 + \bar{\mu}^2) }{n^2\rho}.
\end{equation}
Take expectation over on both sides over $\bar{\mu} | E_1$ (notice that $\bar{\mu}$ and $E_1$ are independent) we obtain an upper and bound of the form $\frac{C ( \tau^2 + \mu^2 + \text{Var}(\bar{\mu})) }{n^2\rho}$ for constant $C=4$ and $C=1/9$ respectively.    

By Assumption (c), $\text{Var}(\bar{\mu})\leq \sigma^2/n$.   This gives rise to an upper bound 
\begin{align}
 \E[(\hat{\mu}- \bar{\mu})^2  | E_1] \leq  \frac{4 ( \tau^2 + \mu^2) }{n^2\rho} + \frac{4\sigma^2}{n^3\rho}.
\end{align}
Note that under $E_1^c$, we have
\begin{align}
\E[(\hat{\mu}^\tau- \bar{\mu})^2| E_1^c]= \E[(\sum_i[Y_i]_{[-\tau,\tau]} + Z_2 -  \frac{\sum_i Y_i }{n })^2| E_1^c] \leq  n\tau^2 + \frac{\tau^2}{\rho}+ \mu^2 + \frac{\sigma^2}{n}.
\end{align}
It follows that  if we choose $\delta \leq 1/n^3$ (under the assumption $\rho < n$)
\begin{align}
 \E[(\hat{\mu} - \bar{\mu})^2 ] \leq  (1-\delta)  \E[(\hat{\mu}^\tau- \bar{\mu})^2  | E_1]  + \delta \E[(\hat{\mu}^\tau - \bar{\mu})^2| E_1^c]=  O(\frac{( \tau^2 + \mu^2 +\sigma^2/n) }{n^2\rho} ),
\end{align}

Moreover, it is clear that by a subgaussian concentration bound, we can prove a high probability error bound (in the $\tau\geq B$ regime) which says that with probability $1-2\delta$,
\begin{align}
	(\hat{\mu}^\tau- \bar{\mu})^2  \leq \frac{8( \tau^2 + \mu^2 \sigma^2 / n)  \log(4/\delta)}{n^2\rho}. \label{eq:high_prob_bound_stage1}
\end{align}
 To get a matching lower bound, consider a particular $\mathcal P$ such that
 $\Pr(Y =\mu)  = 1$, thus $\text{Var}(\bar{\mu}) \geq 0$ and 
\begin{align}
\max_{\mathcal P}\E[(\hat{\mu}^\tau(Y) - \bar{\mu})^2 ]  \geq \frac{(1-\delta)( \tau^2 + \mu^2) }{9n^2\rho}.
\end{align}

Now let's consider the case when $\tau < B$.  We can still start from \eqref{eq:useful_decomp}. Observe that we can still use \eqref{eq:square_err_sandwich} to bound $(**)$, in fact, we can obtain the same $\E[(**)^2]  \geq \frac{(1-\delta)( \tau^2 + \mu^2) }{9n^2\rho}$ for any $\tau< B$ too.
It remains to construct a lower bound $(*)$ for  using the same family of distributions. The idea is that $(*)$ will introduce additional bias that is not vanishing even if $n\rightarrow \infty$.

Again, we will consider a trivial distribution where $\Pr(Y =\mu)  = 1$.  Assume $\mu>\tau$, then
\begin{align}
\max_{\mathcal P}\E\left[\left|\frac{\sum_i[Y_i]_{[-\tau,\tau]} - Y_i }{n+Z_1} \right|^2 \right]&\geq\max_{\mu >\tau}\E[(\frac{n(\mu-\tau)_+}{n+ Z_1})^2]\\
& \geq  \max_{\mu >\tau}\Pr(E_1)(\frac{n(\mu-\tau)_+}{3n/2})^2= \frac{4(1-\delta)\max^2(B-\tau, 0)}{9},
\end{align}
where we applied the $|Z_1| \leq n/2$ which is implied by $E_1$, for which we will choose $\delta < 1/2$. 

Clearly, the above lower bound says that if we stick with $T=1$, one cannot gain anything by choosing $\tau< B$ in terms of the max error. 

Next, we will analyze the algorithm with $T=2$. We will set $\tau_1 = B$ and $\tau_2 = O\left(\max\{\frac{B}{n\sqrt{\rho}},\sigma\} \sqrt{\log(n)}\right)$.

The second boosting round will estimate
\begin{align}
\hat{\mu}_2 = \hat{\mu}_1 + \frac{\sum_i [Y_i - \hat{\mu}_1]_{[-\tau_2,\tau_2]} + Z_3}{n+Z_1}.
\end{align}
where $Z_3 \sim \cN(0,\sigma_3^2)$

Besides $E_1$, we will further consider the following high probability events.

($E_2$) Following \eqref{eq:high_prob_bound_stage1},with probability $1-2\delta$,  $\left|\hat{\mu}_1-\bar{\mu}\right|\leq \frac{CB\sqrt{\log(4/\delta)}}{n\sqrt{\rho}}$.

($E_3$) With probability $1-\delta$ for all $i=1,...,n$, $|Y_i - \mu| \leq  \sigma\sqrt{2\log(2n /\delta)}$.   

($E_4$)Again by subgaussian concentration,  with probability $1-\delta$, $|\bar{\mu} - \mu| \leq \sigma\sqrt{2\log(2/\delta)}$.  

Thus with by triangle inequality,
with probability $1-4\delta$,  
$|\hat{\mu}_1 -Y_i| \leq  \max\{\frac{B}{n\sqrt{\rho}},\sigma\} \sqrt{C\log(2n/\delta)}$  for a constant $C$.  

Let event $E = E_1\cap E_2 \cap E_3 \cap E_4$. Under event $E$, when $\tau_2$ is set with such a bound, clipping will not happen for any data point and
\begin{align}
\notag \hat{\mu}_2  - \bar{\mu}&= \hat{\mu}_1   -   \frac{\sum_{i=1}^n (- \hat{\mu}_1)}{n+Z_1}+ \frac{\sum_i Y_i + Z_3}{n+Z_1} - \bar{\mu}\\
\notag &=  \frac{Z_1 \hat{\mu}_1}{n+Z_1} + \frac{Z_3 + Z_1\bar{\mu}}{n+Z_1} \\
&= \frac{Z_1 (\hat{\mu}_1-\bar{\mu})}{n+Z_1} +\frac{Z_3 + 2Z_1\bar{\mu} }{n+Z_1}  
\end{align}

It follows that
\begin{align}
\E\left[(\hat{\mu}_2  - \bar{\mu})^2  | E \right] &\leq  2\E\left[( \frac{Z_1 (\hat{\mu}_1-\bar{\mu})}{n+Z_1})^2  \middle| E \right] +  2\E\left[(\frac{Z_3 + 2Z_1\bar{\mu} }{n+Z_1}  )^2 \middle| E \right] \\
&= O\left(\frac{B^2+\mu^2}{n^4 \rho^2} + \frac{\sigma^2 \log(n/\delta)+ \mu^2}{n^2\rho}\right).
\end{align}
Under the complement event $E^c$, we apply a trivial bound that can be enforced by the algorithm (clipping the estimate at $[-B,B]$ knowing that $\mu\in[-B,B]$)
\begin{align}
\E[(\hat{\mu}_2- \bar{\mu})^2| E^c]= 4B^2.
\end{align}

Choose $\delta = \frac{1}{n^4\rho^2}$, by the law of total expectation, we get
\begin{align}
\E\left[(\hat{\mu}_2  - \bar{\mu})^2\right]  =O\left(\frac{B^2+\mu^2}{n^4 \rho^2} + \frac{\sigma^2 \log(n\rho)+ \mu^2}{n^2\rho}\right),
\end{align}
which completes the proof.
\end{proof}

\subsubsection{Robustness to outliers via clipping and boosting --- a fixed-point iteration persepctive} \label{s.robustness}

An additional benefit of boosting is robustness to outliers. For simplicity, here we analyze the non-private version of BoostedAdaSSP without noise addition. The reason for this simplification is that the robustness of BoostedAdaSSP can be revealed even without noise. Adding noise for privacy entails straightforward modification to the analysis below.

We consider a contaminated dataset with $n$ data points $Y_1,...,Y_n$. Among them  $n-k$ are in-liers  drawn i.i.d. from $\mathcal P$ - an arbitrary distribution supported on $[\mu-\sigma,\mu+\sigma]$ with mean $\mu$, and $k$ are outliers with a value of $B$. Without loss of generality, we assume the in-liers are the first $n-k$.

A non-boosted algorithm will have to add noise proportional to $B$ to the non-private estimator given by $\hat{\mu} = \frac{kB + \sum_{i=1}^{n-k}Y_i}{n}$. The MSE for estimating $\mu:= \E_\mathcal P Y$ is
\begin{equation}\label{eq:mse_empiricalmean}
\E(\hat{\mu}-\mu)^2 =  \frac{k^2(B-\mu)^2}{n^2} + \frac{\sigma^2(n-k)}{2n^2}. 
\end{equation}
Note that  $|B-\mu|$ can be arbitrarily large. The estimator is thus vulnerable to outliers.

For the boosted version of the algorithm, we fix a relatively small clipping threshold $\tau$ through many iterations until convergence. The solution can be viewed as iteratively solving the nonlinear equation
\begin{equation}\label{eq:nonlinear_equation}
\sum_{i=1}^n   (Y_i -\mu)_{[-\tau,\tau]} = 0,
\end{equation}
because algorithmically, BoostedAdaSSP iteratively simulates the following fixed point equation until convergence:
\begin{equation}\label{eq:fixedpoint_iter}
\hat{\mu}_{t+1} = \hat{\mu}_t + \frac{1}{n}\sum_{i=1}^n   (Y_i -\hat{\mu}_t)_{[-\tau,\tau]}.
\end{equation}

Theorem \ref{theorem:robustness} explains why solution to \eqref{eq:nonlinear_equation}, to which the noiseless BoostedAdaSSP algorithm converges, is robust to outliers.
\begin{theorem}\label{theorem:robustness}
As $\tau\rightarrow 0$, the solution to \eqref{eq:nonlinear_equation} converges to $\mathrm{Median}(Y_{1:n})$.
For a constant $\tau$, the solution to \eqref{eq:nonlinear_equation} minimizes the Huber loss with radius $\tau$.
\end{theorem}

\begin{proof}[Proof of Theorem \ref{theorem:robustness}]
	When $\tau\rightarrow 0$, we may rewrite the equation as $\lim_{\tau\rightarrow 0}\frac{1}{\tau}\sum_{i=1}^n   (Y_i -\mu)_{[-\tau,\tau]} = 0$. Then, observe that $\lim_{\tau\rightarrow 0}\frac{1}{\tau}(\mu-Y_i)_{[-\tau,\tau]}\in \partial |\mu - Y_i|$, and a valid solution certifying the sub-gradient optimality condition for $\min_{\mu} \sum_i|Y_i - \mu|$ is the sample median.
	For the second statement, observe that  $(\mu - Y_i)_{[-\tau,\tau]} = \frac{\partial}{\partial \mu} \textrm{Huber}_\tau(\mu- Y_i) $. A fixed point satisfying the optimality condition is a minimizer.
\end{proof}

Minimizers of the Huber loss functions is among the most well-known robust M-estimators studied in the classical literature. In the regime when $\hat{\mu}_t$ is $\tau$-away from the true support, the iteration moves towards $\mu$ as long as the outlier is less than half of the data. Assume $2\sigma <\tau < B-\mu$, once $\hat{\mu}_t\in[\mu-\tau,\mu+\tau]$  the in-lier data will no longer be clipped, and the fixed point equation becomes 
$
\sum_{i=1}^{n-k} (Y_i   -\hat{\mu})   +  k(B-\hat{\mu})_{[-\tau,\tau]}  = 0,
$
which gives
$
\hat{\mu}_{\infty} = \frac{\sum_{i=1}^{n-k} Y_i + k(B-\hat{\mu}_{\infty} )_{[-\tau,\tau] }}{n-k},
$
with an MSE error bound of 
$\E(\hat{\mu} _{\infty} -\mu)^2\leq  \frac{2k^2\tau^2}{(n-k)^2} + \frac{2\sigma^2}{(n-k)}.$

Compared to the error of empirical mean estimator \eqref{eq:mse_empiricalmean}, the BoostedAdaSSP with a moderate $\tau$ and large $T$  converges to a solution that is significantly more robust to outliers.

\end{document}